\documentclass[11pt]{article}

\usepackage{fullpage,times}

\usepackage{caption,subcaption}

\usepackage{amsthm,amsfonts,amsmath,amssymb,epsfig,color,float,graphicx,verbatim}
\usepackage{algorithm,algorithmic}

\newtheorem{theorem}{Theorem}

\newtheorem{lemma}{Lemma}

\newtheorem{corollary}{Corollary}

\newcommand{\reals}{\mathbb{R}}
\newcommand{\E}{\mathbb{E}}

\newcommand{\be}{\mathbf{e}}
\newcommand{\bx}{\mathbf{x}}
\newcommand{\bw}{\mathbf{w}}
\newcommand{\bg}{\mathbf{g}}

\newcommand{\bu}{\mathbf{u}}
\newcommand{\bv}{\mathbf{v}}
\newcommand{\bz}{\mathbf{z}}
\newcommand{\bc}{\mathbf{c}}

\newcommand{\bs}{\mathbf{s}}

\newcommand{\Ocal}{\mathcal{O}}

\newcommand{\norm}[1]{\|#1\|}
\newcommand{\inner}[1]{\langle#1\rangle}

\newcommand{\secref}[1]{Section~\ref{#1}}
\newcommand{\subsecref}[1]{Subsection~\ref{#1}}

\renewcommand{\eqref}[1]{Eq.~(\ref{#1})}
\newcommand{\lemref}[1]{Lemma~\ref{#1}}

\newcommand{\thmref}[1]{Theorem~\ref{#1}}

\DeclareMathOperator{\Tr}{Tr}

\title{Fast Stochastic Algorithms for SVD and PCA:\\ Convergence Properties and Convexity}
\author{Ohad Shamir\\Weizmann Institute of Science\\\texttt{ohad.shamir@weizmann.ac.il}}
\date{}

\begin{document}

\maketitle

\begin{abstract}
We study the convergence properties of the VR-PCA algorithm introduced by
\cite{shamir2015stochastic} for fast computation of leading singular
vectors. We prove several new results, including a formal analysis of a
block version of the algorithm, and convergence from random initialization. We also make a few observations of independent interest, such as how
pre-initializing with just a single exact power iteration can significantly improve the runtime of stochastic methods, and what are the convexity and
non-convexity properties of the underlying optimization problem.
\end{abstract}

\section{Introduction}

We consider the problem of recovering the top $k$ left singular
vectors of a $d\times n$ matrix $X=(\bx_1,\ldots,\bx_n)$, where $k\ll d$.
This is equivalent to recovering the top $k$ eigenvectors of $XX^\top$, or
equivalently, solving the optimization problem
\begin{equation}\label{eq:optproblem}
\min_{W\in\reals^{d\times k}:W^\top
W=I}-W^\top\left(\frac{1}{n}\sum_{i=1}^{n}\bx_i\bx_i^\top\right)W.
\end{equation}
This is one of the most fundamental matrix computation problems, and has
numerous uses (such as low-rank matrix approximation and principal component
analysis).

For large-scale matrices $X$, where exact eigendecomposition is infeasible,
standard deterministic approaches are based on power iterations or variants
thereof (e.g. the Lanczos method) \cite{golub2012matrix}. Alternatively, one
can exploit the structure of \eqref{eq:optproblem} and apply stochastic
iterative algorithms, where in each iteration we update a current $d\times k$
matrix $W$ based on one or more randomly-drawn columns $\bx_i$ of $X$. Such
algorithms have been known for several decades
(\cite{krasulina1969method,oja1982simplified}), and enjoyed renewed interest
in recent years, e.g.
\cite{ACLS12,balsubramani2013fast,arora2013stochastic,hardt2014noisy,de2015global}.
Another stochastic approach is based on random projections, e.g.
\cite{halko2011finding,woodruff2014sketching}.

Unfortunately, each of these algorithms suffer from a
different disadvantage: The deterministic algorithms are accurate (runtime
logarithmic in the required accuracy $\epsilon$, under an eigengap condition), but require a full pass
over the matrix for each iteration, and in the worst-case many such passes
would be required (polynomial in the eigengap). On the other hand, each iteration of the stochastic
algorithms is cheap, and their number is independent of the size of the
matrix, but on the flip side, their noisy stochastic nature means they are
not suitable for obtaining a high-accuracy solution (the runtime scales
polynomially with $\epsilon$).

Recently, \cite{shamir2015stochastic} proposed a new practical algorithm,
VR-PCA, for solving \eqref{eq:optproblem}, which has a
``best-of-both-worlds'' property: The algorithm is based on cheap stochastic
iterations, yet the algorithm's runtime is logarithmic in the required
accuracy $\epsilon$. More precisely, for the case $k=1$, $\bx_i$ of bounded norm, and when there is an
eigengap of $\lambda$ between the first and second leading eigenvalues
of the covariance matrix $\frac{1}{n}XX^\top$, the required runtime was shown to be on the order of
\begin{equation}\label{eq:runtimeme}
d\left(n+\frac{1}{\lambda^2}\right)\log\left(\frac{1}{\epsilon}\right).
\end{equation}
The algorithm is therefore suitable for obtaining high accuracy solutions
(the dependence on $\epsilon$ is logarithmic), but essentially at the cost of
only $\Ocal(\log(1/\epsilon))$ passes over the data. The algorithm is based
on a recent variance-reduction technique designed to speed up stochastic
algorithms for \emph{convex} optimization problems
(\cite{johnson2013accelerating}), although the optimization problem in
\eqref{eq:optproblem} is inherently non-convex. See \secref{sec:alg} for a
more detailed description of this algorithm, and \cite{shamir2015stochastic}
for more discussions as well as empirical results.

The results and analysis in \cite{shamir2015stochastic} left several issues
open. For example, it is not clear if the quadratic dependence on $1/\lambda$
in \eqref{eq:runtimeme} is necessary, since it is worse than the linear (or
better) dependence that can be obtained with the deterministic algorithms
mentioned earlier, as well as analogous results that can be obtained with
similar techniques for convex optimization problems (where $\lambda$ is the strong convexity parameter). Also, the analysis was
only shown for the case $k=1$, whereas often in practice, we may want to
recover $k>1$ singular vectors simultaneously. Although
\cite{shamir2015stochastic} proposed a variant of the algorithm for that
case, and studied it empirically, no analysis was provided. Finally, the
convergence guarantee assumed that the algorithm is
initialized from a point closer to the optimum than what is attained with
standard random initialization. Although one can use some other, existing
stochastic algorithm to do this ``warm-start'', no end-to-end analysis of the
algorithm, starting from random initialization, was provided.

In this paper, we study these and related questions, and make the following
contributions:
\begin{itemize}
  \item We propose a variant of VR-PCA to handle the $k>1$ case, and
      formally analyze its convergence (\secref{sec:alg}). The extension to
      $k>1$ is non-trivial, and requires tracking the evolution of the
      subspace spanned by the current solution at each iteration.
  \item In \secref{sec:warmstart}, we study the convergence of VR-PCA
      starting from a random initialization. And show that with a slightly
      smarter initialization -- essentially, random initialization followed
      by a \emph{single} power iteration -- the convergence results can be
      substantially improved. In fact, a similar initialization scheme
      should assist in the convergence of other stochastic algorithms for
      this problem, as long as a single power iteration can be performed.
  \item In \secref{sec:convex}, we study whether functions similar to
      \eqref{eq:optproblem} have hidden convexity properties, which would
      allow applying existing convex optimization tools as-is, and improve
      the required runtime. For the $k=1$ case, we show that this is in
      fact true: Close enough to the optimum, and on a suitably-designed
      convex set, such a function is indeed $\lambda$-strongly convex.
      Unfortunately, the distance from the optimum has to be
      $\Ocal(\lambda)$, and this precludes a better runtime in most
      practical regimes. However, it still indicates that a better
      runtime and dependence on $\lambda$ should be possible.
\end{itemize}

\section{Some Preliminaries and Notation}\label{sec:prelim}

We consider a $d\times n$ matrix $X$ composed of $n$ columns
$(\bx_1,\ldots,\bx_n)$, and let
\[
A=\frac{1}{n}XX^\top = \frac{1}{n}\sum_{i=1}^{n}\bx_i\bx_i^\top.
\]
Thus, \eqref{eq:optproblem} is equivalent to finding the $k$ leading
eigenvectors of $A$.

We generally use bold-face letters to denote vectors, and capital letters to
denote matrices. We let $\Tr(\cdot)$ denote the trace of a matrix,
$\norm{\cdot}_F$ to denote the Frobenius norm, and $\norm{\cdot}_{sp}$ to
denote the spectral norm. A symmetric $d\times d$ matrix $B$ is positive
semidefinite, if $\inf_{\bz\in \reals^d}\bz^\top B \bz\geq 0$. $A$ is
positive definite if the inequality is strict. Following standard notation,
we write $B\succeq 0$ to denote that $A$ is positive semidefinite, and
$B\succeq C$ if $B-C\succeq 0$. $B\succ 0$ means that $B$ is positive
definite.

A twice-differentiable function $F$ on a subset of $\reals^d$ is convex, if
its Hessian is alway positive semidefinite. If it is always positive
definite, and $\succ \lambda I$ for some $\lambda>0$, we say that the
function is $\lambda$-strongly convex. If the Hessian is always $\prec sI$
for some $s\geq 0$, then the function is $s$-smooth.

\section{The VR-PCA Algorithm and a Block Version}\label{sec:alg}

We begin by recalling the algorithm of \cite{shamir2015stochastic} for the
$k=1$ case (Algorithm \ref{alg:algvec}), and then discuss its generalization
for $k>1$.

\begin{center}
\begin{minipage}{0.7\textwidth}
\begin{algorithm}[H]
\caption{VR-PCA: Vector version ($k=1$)} \label{alg:algvec}
\begin{algorithmic}[1]
\STATE \textbf{Parameters:} Step size $\eta$, epoch length $m$ \STATE
\textbf{Input:} Data matrix $X=(\bx_1,\ldots,\bx_n)$; Initial unit vector
$\tilde{\bw}_0$ \FOR{$s=1,2,\ldots$}
  \STATE $\tilde{\bu}=\frac{1}{n}\sum_{i=1}^{n}\bx_i\left(\bx_i^\top \tilde{\bw}_{s-1}\right)$
  \STATE $\bw_0=\tilde{\bw}_{s-1}$
  \FOR{$t=1,2,\ldots,m$}
    \STATE Pick $i_t\in \{1,\ldots,n\}$ uniformly at random
    \STATE $\bw'_{t}=\bw_{t-1}+\eta\left(\bx_{i_t}\left(\bx_{i_t}^\top\bw_{t-1}-\bx_{i_t}^\top\tilde{\bw}_{s-1}\right)+\tilde{\bu}\right)$
    \STATE $\bw_{t}=\frac{1}{\norm{\bw'_t}}\bw'_{t}$
  \ENDFOR
  \STATE $\tilde{\bw}_{s}=\bw_m$
\ENDFOR
\end{algorithmic}
\end{algorithm}
\end{minipage}
\end{center}

The basic idea of the algorithm is to perform stochastic updates using
randomly-sampled columns $\bx_i$ of the matrix, but interlace them with
occasional exact power iterations, and use that to gradually reduce the
variance of the stochastic updates. Specifically, the algorithm is split into
epochs $s=1,2,\ldots$, where in each epoch we do a single exact power
iteration with respect to the matrix $A$ (by computing $\tilde{\bu}$), and then perform $m$ stochastic
updates, which can be re-written as
\[
\bw'_{t} = (I+\eta A)\bw_{t-1}+\eta\left(\bx_{i_t}\bx_{i_t}^\top-A\right)\left(\bw_{t-1}-\tilde{\bw}_{s-1}\right)~~,~~
\bw_{t} = \frac{1}{\norm{\bw'_{t}}}\bw_{t},
\]
The first term is essentially a power iteration (with a finite step size
$\eta$), whereas the second term is zero-mean, and with variance dominated by
$\norm{\bw_{t-1}-\tilde{\bw}_{s-1}}^2$. As the algorithm progresses,
$\bw_{t-1}$ and $\tilde{\bw}_{s-1}$ both converge toward the same optimal
point, hence $\norm{\bw_{t-1}-\tilde{\bw}_{s-1}}^2$ shrinks, eventually
leading to an exponential convergence rate.

To handle the $k>1$ case (where more than one eigenvector should be recovered), one simple technique is deflation, where we recover the leading eigenvectors $\bv_1,\bv_2,\ldots,\bv_k$ one-by-one, each time using the $k=1$ algorithm. However, a disadvantage
of this approach is that it requires a positive eigengap between all top $k$
eigenvalues, otherwise the algorithm is not guaranteed to converge. Thus, an algorithm which simultaneously recovers all $k$ leading eigenvectors is  preferable.

We will study a block version of Algorithm \ref{alg:algvec}, presented as Algorithm
\ref{alg:algblock}. It is mostly a straightforward generalization (similar to
how power iterations are generalized to orthogonal iterations), where the
$d$-dimensional vectors $\bw_{t-1}, \tilde{\bw}_{s-1}, \bu$ are replaced by
$d\times k$ matrices $W_{t-1},\tilde{W}_{s-1},\tilde{U}$, and normalization
is replaced by orthogonalization\footnote{The normalization $W_t =
W'_{t}\left(W^{'\top}_{t}W'_{t}\right)^{-1/2}$ ensures that $W_t$ has
orthonormal columns. We note that in our analysis, $\eta$ is chosen
sufficiently small so that $W^{'\top}_{t}W'_{t}$ is always invertible, hence
the operation is well-defined.}. Indeed, Algorithm \ref{alg:algvec} is
equivalent to Algorithm \ref{alg:algblock} when $k=1$. The main twist in
Algorithm \ref{alg:algblock} is that instead of using
$\tilde{W}_{s-1},\tilde{U}$ as-is, we perform a unitary transformation (via
the $k\times k$ orthogonal matrix $B_{t-1}$) which maximally aligns them with
$W_{t-1}$. Note that $B_{t-1}$ is a $k\times k$ matrix, and since $k$ is
assumed to be small, this does not introduce significant computational
overhead.

\begin{center}
\begin{minipage}{0.8\textwidth}
\begin{algorithm}[H]
\caption{VR-PCA: Block version} \label{alg:algblock}
\begin{algorithmic}
\STATE \textbf{Parameters:} Rank $k$, Step size $\eta$, epoch length $m$
\STATE \textbf{Input:} Data matrix $X=(\bx_1,\ldots,\bx_n)$; Initial $d\times
k$ matrix $\tilde{W}_0$ with orthonormal columns\FOR{$s=1,2,\ldots$}
  \STATE $\tilde{U}=\frac{1}{n}\sum_{i=1}^{n}\bx_i\left(\bx_i^\top \tilde{W}_{s-1}\right)$
  \STATE $W_0=\tilde{W}_{s-1}$
  \FOR{$t=1,2,\ldots,m$}
    \STATE $B_{t-1}=VU^\top$, where $USV^\top$ is an SVD decomposition of
    $W_{t-1}^\top \tilde{W}_{s-1}$
    \STATE $~~\rhd~~$ Equivalent to $B_{t-1} = \arg\min_{B^\top B=I}\norm{W_{t-1}-\tilde{W}_{s-1}B}_F^2$
    \STATE Pick $i_t\in \{1,\ldots,n\}$ uniformly at random
    \STATE $W'_{t}=W_{t-1}+\eta\left(\bx_{i_t}\left(\bx_{i_t}^\top W_{t-1}-\bx_{i_t}^\top\tilde{W}_{s-1}B_{t-1}\right)+\tilde{U}B_{t-1}\right)$
    \STATE $W_{t}=W'_{t}\left(W^{'\top}_{t}W'_{t}\right)^{-1/2}$
  \ENDFOR
  \STATE $\tilde{W}_{s}=W_m$
\ENDFOR
\end{algorithmic}
\end{algorithm}
\end{minipage}
\end{center}

We now turn to provide a formal analysis of Algorithm \ref{alg:algblock},
which directly generalizes the analysis of Algorithm \ref{alg:algvec} given
in \cite{shamir2015stochastic}:

\begin{theorem}\label{thm:main}
  Define the $d\times d$ matrix $A$ as $\frac{1}{n}XX^\top=\frac{1}{n}\sum_{i=1}^{n}\bx_i\bx_i^\top$, and let $V_k$ denote the $d\times k$
  matrix composed of the eigenvectors corresponding to the largest $k$ eigenvalues. Suppose that
  \begin{itemize}
    \item $\max_i\norm{\bx_i}^2\leq r$ for some $r>0$.
    \item $A$ has eigenvalues $s_1>s_2\geq\ldots\geq s_d$, where
        $s_k-s_{k+1}=\lambda$ for some $\lambda>0$.
    \item $k-\norm{V_k^\top \tilde{W}_0}_F^2\leq \frac{1}{2}$.
  \end{itemize}
  Let $\delta,\epsilon\in (0,1)$ be fixed. If we run the algorithm with any epoch length parameter $m$ and step size $\eta$, such that
  \begin{equation}\label{eq:thmcondme}
\eta \leq \frac{c\delta^2}{r^2}\lambda~~~~,~~~~
m\geq \frac{c'\log(2/\delta)}{\eta \lambda}
~~~~,~~~~km\eta^2r^2+rk\sqrt{m\eta^2\log(2/\delta)}\leq c''
  \end{equation}
  (where $c,c',c''$ designate certain positive numerical constants), and for $T=\left\lceil\frac{\log(1/\epsilon)}{\log(2/\delta)}\right\rceil$ epochs,
  then with probability at least $1-\lceil\log_2(1/\epsilon)\rceil\delta$, it holds that
  \[
  k-\norm{V_k^\top \tilde{W}_T}_F^2 ~\leq~ \epsilon.
  \]
\end{theorem}

For any orthogonal $W$, $k-\norm{V_k^\top W}_F^2$ lies between $0$ and $k$,
and equals $0$ when the column spaces of $V_k$ and $W$ are the same (i.e.,
when $W$ spans the $k$ leading singular vectors). According to the theorem,
taking appropriate\footnote{Specifically, we can take
$m=c'\log(2/\delta)/\eta \lambda$ and $\eta=a\delta^2/r^2\lambda$, where $a$
is sufficiently small to ensure that the first and third condition in
\eqref{eq:thmcondme} holds. It can be verified that it's enough to take
$a=\min\left\{c,\frac{c''}{4\delta^2 c k \log(2/\delta)},\frac{1}{4\delta^2
c}\left(\frac{c''}{k\log(2/\delta)}\right)^2\right\}$.}
$\eta=\Theta(\lambda/(kr)^2)$, and $m=\Theta((rk/\lambda)^2)$, the algorithm
converges with high probability to a high-accuracy approximation of $V_k$.
Moreover, the runtime of each epoch of the algorithm equals
$\Ocal(mdk^2+dnk)$. Overall, we get the following corollary:

\begin{corollary}\label{corr:main}
  Under the conditions of \thmref{thm:main}, there exists an algorithm
  returning $\tilde{W}_T$ such that $k-\norm{V_k^\top \tilde{W}_T}_F^2\leq
  \epsilon$ with arbitrary constant accuracy, in runtime
  $\Ocal\left(dk(n+\frac{r^2k^3}{\lambda^2})\log(1/\epsilon)\right)$.
\end{corollary}
This runtime bound is the same\footnote{\cite{shamir2015stochastic} showed
that it's possible to further improve the runtime for sparse $X$, replacing
$d$ by the average column sparsity $d_s$. This is done by maintaining
parameters in an implicit form, but it's not clear how to implement a similar
trick in the block version, where $k>1$.} as that of
\cite{shamir2015stochastic} for $k=1$.

The proof of \thmref{thm:main} appears in \subsecref{subsec:proofmain}, and
relies on a careful tracking of the evolution of the potential function
$k-\norm{V_k^\top\tilde{W}_t}_F^2$. An important challenge compared to the
$k=1$ case is that the matrices $W_{t-1}$ and $\tilde{W}_{s-1}$ do not necessarily
become closer over time, so the variance-reduction intuition discussed
earlier no longer applies. However, the column space of $W_{t-1}$ and
$\tilde{W}_{s-1}$ do become closer, and this is utilized by introducing the
transformation matrix $B_{t-1}$. We note that although $B_{t-1}$ appears
essential for our analysis, it isn't clear that using it is necessary in
practice: In \cite{shamir2015stochastic}, the suggested block algorithm was
Algorithm \ref{alg:algblock} with $B_{t-1}=I$, which seemed to work well in
experiments. In any case, using this matrix doesn't affect the overall
runtime beyond constants, since the additional runtime of computing and using
this matrix ($\Ocal(dk^2)$) is the same as the other computations performed
at each iteration.

A limitation of the theorem above is the assumption that the initial point
$\tilde{W}_0$ is such that $k-\norm{V_k^\top \tilde{W}_0}_F^2 \leq
\frac{1}{2}$. This is a non-trivial assumption, since if we initialize the
algorithm from a random $d\times \Ocal(1)$ orthogonal matrix $\tilde{W}_0$, then
with overwhelming probability, $\norm{V_k^\top
\tilde{W}_0}_F^2=\Ocal(1/d)$. However, experimentally the algorithm seems to work well even with random initialization \cite{shamir2015stochastic}. Moreover, if we are interested in a theoretical guarantee, one simple solution is to warm-start the algorithm with a purely stochastic algorithm for this problem (such as \cite{de2015global,hardt2014noisy,balsubramani2013fast}), with runtime guarantees on getting such a $\tilde{W}_0$. The idea is that $\tilde{W}_0$ is only required to approximate $V_k$ up to constant accuracy, so purely
stochastic algorithms (which are good in obtaining a low-accuracy solution)
are quite suitable. In the next section, we further delve into these issues, and show that in our setting such algorithms in fact can be substantially improved.

\section{Warm-Start and the Power of a Power Iteration}\label{sec:warmstart}

In this section, we study the runtime required to compute a starting point
satisfying the conditions of \thmref{thm:main}, starting from a random
initialization. Combined with \thmref{thm:main}, this gives us an end-to-end
analysis of the runtime required to find an $\epsilon$-accurate solution,
starting from a random point. For simplicity, we will only discuss the case
$k=1$, i.e. where our goal is to compute the single leading eigenvector
$\bv_1$, although our observations can be generalized to $k>1$. In the $k=1$ case, \thmref{thm:main} kicks in once we find a vector $\bw$
satisfying $\inner{\bv_1,\bw}^2\geq \frac{1}{2}$.

As mentioned previously, one way to get such a $\bw$ is to run a purely
stochastic algorithm, which computes the leading eigenvector of a covariance
matrix $\E[\bx\bx^\top]$ given a stream of i.i.d. samples $\bx$. We can easily
use such an algorithm in our setting, by sampling columns from our matrix
$X=(\bx_1,\ldots,\bx_n)$ uniformly at random, and feed to such a stochastic
optimization algorithm, guaranteed to approximate the leading eigenvector of
$\frac{1}{n}\sum_{i=1}^{n}\bx_i\bx_i^\top$.

To the best of our knowledge, the existing iteration complexity guarantees
for such algorithms (assuming the norm constraint $r\leq 1$ for simplicity)
scale at least\footnote{For example, this holds for \cite{de2015global},
although the bound only guarantees the existence of \emph{some} iteration
which produces the desired output. The guarantee of
\cite{balsubramani2013fast} scale as $d^2/\lambda^2$, and the guarantee of
\cite{hardt2014noisy} scales as $d/\lambda^3$ in our setting.} as $d/\lambda^2$.
Since the runtime of each iteration is $\Ocal(d)$, we get an overall runtime
of $\Ocal((d/\lambda)^2)$.

The dependence on $d$ in the iteration bound stems from the fact that with a
random initial unit vector $\bw_0$, we have $\inner{\bv_1,\bw_0}^2\approx
\frac{1}{d}$. Thus, we begin with a vector almost orthogonal to the leading
eigenvector $\bv_1$ (depending on $d$). In a purely stochastic setting, where only noisy information is available, this necessitates conservative updates at first, and in all the analyses we are aware of, the number of iterations appear to necessarily scale at least linearly with $d$.

However, it turns out that in our setting, with a finite matrix $X$, we can perform a smarter initialization: Sample $\bw$ from the standard Gaussian distribution
on $\reals^d$, perform a \emph{single} power iteration w.r.t. the covariance matrix $A=\frac{1}{n}XX^\top$, i.e. $\bw_0 = A\bw/\norm{A\bw}$,
and initialize from $\bw_0$. For such a procedure, we have the following
simple observation:

\begin{lemma}\label{lem:smart}
  For $\bw_0$ as above, it holds for any $\delta$ that with probability at
  least $1-\frac{1}{d}-\delta$,
  \[
    \inner{\bv_1,\bw_0}^2 ~\geq~ \frac{\delta^2}{12\log(d)~\text{nrank}(A)},
  \]
  where $\text{nrank}(A)= \frac{\norm{A}_{F}^2}{\norm{A}_{sp}^2}$ is the
  numerical rank of $A$.
\end{lemma}
The numerical rank (see e.g. \cite{rudelson2007sampling}) is a relaxation of
the standard notion of rank: For any $d\times d$ matrix $A$,
$\text{nrank(A)}$ is at most the rank of $A$ (which in turn is at most $d$).
However, it will be small even if $A$ is just close to being low-rank. In
many if not most machine learning applications, we are interested in matrices
which tend to be approximately low-rank, in which case $\text{nrank(A)}$ is
much smaller than $d$ or even a constant. Therefore, by a single power
iteration, we get an initial point $\bw_0$ for which $\inner{\bv_1,\bw_0}^2$
is on the order of $1/\text{nrank(A)}$, which can be much larger than the
$1/d$ given by a random initialization, and is never substantially worse.

\begin{proof}[Proof of \lemref{lem:smart}]
  Let $s_1\geq s_2\geq\ldots\geq s_d\geq 0$ be the $d$ eigenvalues of $A$,
  with eigenvectors $\bv_1,\ldots,\bv_d$. We have
  \[
  \inner{\bv_1,\bw_0}^2 = \frac{\inner{\bv_1,A\bw}^2}{\norm{A\bw}^2} = \frac{\left(s_1\inner{\bv_1,\bw}\right)^2}{\left(\sum_{i=1}^{d}s_i\bv_i\inner{\bv_i,\bw}\right)^2}
  = \frac{s_1^2 \inner{\bv_1,\bw}^2}{\sum_{i=1}^{d}s_i^2\inner{\bv_i,\bw}^2}.
  \]
  Since $\bw$ is distributed according to a standard Gaussian distribution,
  which is rotationally symmetric, we can assume without loss of generality
  that $\bv_1,\ldots,\bv_d$ correspond to the standard basis vectors
  $\be_1,\ldots,\be_d$, in which case the above reduces to
  \[
  \frac{s_1^2 w_1^2}{\sum_{i=1}^{d}s_i^2 w_i^2} ~\geq~ \frac{s_1^2}{\sum_{i=1}^{d} s_i^2} \frac{w_1^2}{\max_i w_i^2},
  \]
  where $w_1,\ldots,w_d$ are independent and scalar random variables with a
  standard Gaussian distribution.

  First, we note that $s_1^2$ equals $\norm{A}_{sp}^2$, the spectral norm of $A$, whereas
  $\sum_{i=1}^{d}
  s_i^2$ equals $\norm{A}_{F}^2$, the Frobenius norm of $A$. Therefore,
  $\frac{s_1^2}{\sum_i s_i^2} = \frac{\norm{A}_{sp}^2}{\norm{A}_{F}^2} =
  \frac{1}{\text{nrank}(A)}$, and we get overall that
  \begin{equation}\label{eq:vww}
  \inner{\bv_1,\bw_0}^2 ~\geq~ \frac{1}{\text{nrank}(A)}\frac{w_1^2}{\max_i w_i^2}.
  \end{equation}

  We consider the random quantity $w_1^2/\max_i w_i^2$, and independently
  bound the deviation probability of the numerator and denominator. First, for any $t\geq 0$ we
  have
  \begin{equation}\label{eq:numin}
      \Pr(w_1^2 \leq t) = \Pr(w_1\in [-\sqrt{t},\sqrt{t}]) = \int_{z=-\sqrt{t}}^{\sqrt{t}}\sqrt{\frac{1}{2\pi}}\exp\left(-\frac{z^2}{2}\right) \leq \sqrt{\frac{1}{2\pi}}*2\sqrt{t}
  = \sqrt{\frac{2}{\pi}t}~.
  \end{equation}
  Second, by combining two standard Gaussian concentration results (namely,
  that if $W=\max\{|w_1|,\ldots,|w_d|\}$, then $0\leq \E[W] \leq 2\sqrt{2\log(d)}$, and
  by the Cirelson-Ibragimov-Sudakov inequality, $\Pr(W-\E[W]>t)\leq
  \exp(-t^2/2)$), we get that
  \[
  \Pr(\max_i |w_i|> 2\sqrt{2\log(d)}+t)  \leq \exp(-t^2/2),
  \]
  and therefore
  \begin{equation}
  \Pr(\max_i w_i^2 > (2\sqrt{2\log(d)}+t)^2) \leq \exp(-t/2).\label{eq:denomin}
  \end{equation}
  Combining \eqref{eq:numin} and \eqref{eq:denomin}, with a union bound, we get that for any
  $t_1,t_2\geq 0$, it holds with probability at least
  $1-\sqrt{\frac{2}{\pi}t_1}-\exp(-t_2^2/2)$ that
  \[
  \frac{w_1^2}{\max_i w_i^2} \geq \frac{t_1}{(2\sqrt{2\log(d)}+t_2)^2}.
  \]
  To slightly simplify this for readability, we take $t_2=\sqrt{2\log(d)}$,
  and substitute $\delta=\sqrt{\frac{2}{\pi}t_1}$. This implies that with probability
  at least $1-\delta-1/d$,
  \[
  \frac{w_1^2}{\max_i w_i^2} ~\geq~ \frac{\frac{\pi}{2}\delta^2}{18\log(d)} > \frac{\delta^2}{12\log(d)}.
  \]
  Plugging back into \eqref{eq:vww}, the result follows.
\end{proof}

This result can be plugged into the existing analyses of purely stochastic
PCA/SVD algorithms, and can often improve the dependence on the $d$ factor in
the iteration complexity bounds to a dependence on the numerical rank of $A$.
We again emphasize that this is applicable in a situation where we can
actually perform a power iteration, and not in a purely stochastic setting
where we only have access to an i.i.d. data stream (nevertheless, it would be
interesting to explore whether this idea can be utilized in such a streaming
setting as well).

To give a concrete example of this, we provide a convergence analysis of the
VR-PCA algorithm (Algorithm \ref{alg:algvec}), starting from an arbitrary
initial point, bounding the total number of stochastic iterations required by
the algorithm in order to produce a point satisfying the conditions of
\thmref{thm:main} (from which point the analysis of \thmref{thm:main} takes
over). Combined with \thmref{thm:main}, this analysis also justifies that
VR-PCA indeed converges starting from a random initialization.

\begin{theorem}\label{thm:burn}
  Using the notation of \thmref{thm:main} (where $\lambda$ is the eigengap, $\bv_1$ is the leading eigenvector, and $r=\max_i\norm{\bx_i}^2$),
  and for any $\delta\in (0,\frac{1}{2})$, suppose we run Algorithm
  \ref{alg:algvec} with some initial unit-norm vector $\tilde{\bw}_0$ such that
  \[
  \inner{\bv_1,\tilde{\bw}_0}^2\geq \zeta>0,
  \]
  and a step size $\eta$ satisfying
  \begin{equation}
  \eta\leq \frac{c\delta^2\lambda\zeta^3}{r^2\log^2(2/\delta)}\label{eq:etaupbound}
  \end{equation}
  (for some universal constant $c$). Then with probability at least $1-\delta$, after
  \[
  T ~=~ \left\lfloor\frac{c'\log(2/\delta)}{\eta\lambda\zeta}\right\rfloor
  \]
  stochastic iterations (lines $6-10$ in the pseudocode, where $c'$ is again a universal constant), we get a point $\bw_T$
  satisfying $1-\inner{\bv_1,\bw_T}^2\leq \frac{1}{2}$. Moreover, if $\eta$
  is chosen on the same order as the upper bound in \eqref{eq:etaupbound},
  then
  \[
  T ~=~ \Theta\left(\frac{r^2\log^3(2/\delta)}{\delta^2\lambda^2\zeta^4}\right).
  \]
\end{theorem}
Note that the analysis does not depend on the choice of the epoch size $m$,
and does not use the special structure of VR-PCA (in fact, the technique we
use is applicable to any algorithm which takes stochastic gradient steps to
solve this type of problem\footnote{Although there exist previous analyses of
such algorithms in the literature, they unfortunately do not quite apply to
our algorithm, for various technical reasons.}). The proof of the theorem
appears in \secref{subsec:proofburn}.

Considering $\delta,r$ as a constants, we get that the runtime required by
VR-PCA to find a point $\bw$ such that $1-\inner{\bv_1,\bw_T}^2\leq
\frac{1}{2}$ is $\Ocal(d/\lambda^2\zeta^4)$ where
$\zeta$ is a lower bound on $\inner{\bv_1,\tilde{\bw}_0}^2$. As discussed earlier, if
$\tilde{\bw}_0$ is a result of random initialization followed by a power
iteration (requiring $\Ocal(nd)$ time), and the covariance matrix $A$ has
small numerical rank, then $\zeta=\inner{\bv_1,\tilde{\bw}_0}^2=\tilde
\Omega(1/\log(d))$,
and the runtime is
\[
\Ocal\left(nd+\frac{d}{\lambda^2}\log^4(d)\right)~=~
\Ocal\left(d\left(n+\left(\frac{\log^2(d)}{\lambda}\right)^2\right)\right).
\]
By Corollary \ref{corr:main}, the runtime required by
VR-PCA from that point to get an $\epsilon$-accurate solution is
\[
\Ocal\left(d\left(n+\frac{1}{\lambda^2}\right)\log\left(\frac{1}{\epsilon}\right)\right),
\]
so the sum of the two expressions (which is $d\left(n+\frac{1}{\lambda^2}\right)$ up to log-factors),  represents the total runtime required by the algorithm.

Finally, we note that this bound holds under the reasonable assumption that
the numeric rank of $A$ is constant. If this assumption doesn't hold, $\zeta$
can be as large as $d$, and the resulting bound will have a worse polynomial
dependence on $d$. We suspect that this is due to a looseness in the
dependence on $\zeta=\inner{\bv_1,\tilde{\bw}_0}^2$ in \thmref{thm:burn},
since better dependencies can be obtained, at least for slightly different
algorithmic approaches (e.g.
\cite{balsubramani2013fast,hardt2014noisy,de2015global}). We leave a
sharpening of the bound w.r.t. $\zeta$ as an open problem.

\section{Convexity and Non-Convexity of the Rayleigh Quotient}\label{sec:convex}

As mentioned in the introduction, an intriguing open question is whether the
$d\left(n+\frac{1}{\lambda^2}\right)\log\left(\frac{1}{\epsilon}\right)$
runtime guarantees from the previous sections can be further improved.
Although a linear dependence on $d,n$ seems unavoidable, this is not the case
for the quadratic dependence on $1/\lambda$. Indeed, when using deterministic
methods such as power iterations or the Lanczos method, the dependence on
$\lambda$ in the runtime is only $1/\lambda$ or even $\sqrt{1/\lambda}$
\cite{kuczynski1992estimating}. In the world of convex optimization from
which our algorithmic techniques are derived, the analog of
$\lambda$ is the strong convexity parameter of the function, and again, it is
possible to get a dependence of $1/\lambda$, or even $\sqrt{1/\lambda}$ with
accelerated schemes (see e.g.
\cite{johnson2013accelerating,nitanda2014stochastic,frostig2015regularizing}
in the context of the variance-reduction technique we use). Is it possible to
get such a dependence for our problem as well?

Another question is whether the non-convex problem that we are tackling
(\eqref{eq:optproblem}) is really that non-convex. Clearly, it has a nice
structure (since we can solve the problem in polynomial time), but perhaps it
actually has hidden convexity properties, at least close enough to the
optimal points? We note that \eqref{eq:optproblem} can be ``trivially''
convexified, by re-casting it as an equivalent semidefinite program
\cite{boyd2004convex}. However, that would require optimization over $d\times
d$ matrices, leading to poor runtime and memory requirements. The question
here is whether we have any convexity with respect to the \emph{original}
optimization problem over ``thin'' $d\times k$ matrices.

In fact, the two questions of improved runtime and convexity are closely
related: If we can show that the optimization problem is convex in some
domain containing an optimal point, then we may be able to use fast
stochastic algorithms designed for convex optimization problems, inheriting
their good guarantees.

To discuss these questions, we will focus on the $k=1$ case for simplicity
(i.e., our goal is to find a leading eigenvector of the matrix
$A=\frac{1}{n}XX^\top = \frac{1}{n}\sum_{i=1}^{n}\bx_i\bx_i^\top$), and study
potential convexity properties of the negative Rayleigh quotient,
\[
F_A(\bw) ~=~ -\frac{\bw^\top A\bw}{\norm{\bw}^2} ~=~ \frac{1}{n}\sum_{i=1}^{n}\left(-\frac{\inner{\bw,\bx_i}^2}{\norm{\bw}^2}\right).
\]
Note that for $k=1$, this function coincides with \eqref{eq:optproblem} on
the unit Euclidean sphere, and with the same optimal points, but has the nice property
of being defined on the entire Euclidean space (thus, at least its domain is
convex).

At a first glance, such functions $F_A$ appear to potentially be convex at
some bounded distance from an optimum, as illustrated for instance in the
case where $A=\left(
                           \begin{array}{cc}
                             1 & 0 \\
                             0 & 0 \\
                           \end{array}
                         \right)$ (see Figure \ref{fig:rq}).
Unfortunately, it turns out that the figure is misleading, and in fact the
function is \emph{not} convex almost everywhere:

\begin{figure}[t]
\centering
\begin{minipage}{.45\textwidth}
  \centering
  \vskip -1cm
  \includegraphics[trim=2.5cm 1cm 1cm 1cm, clip=true,scale=0.35]{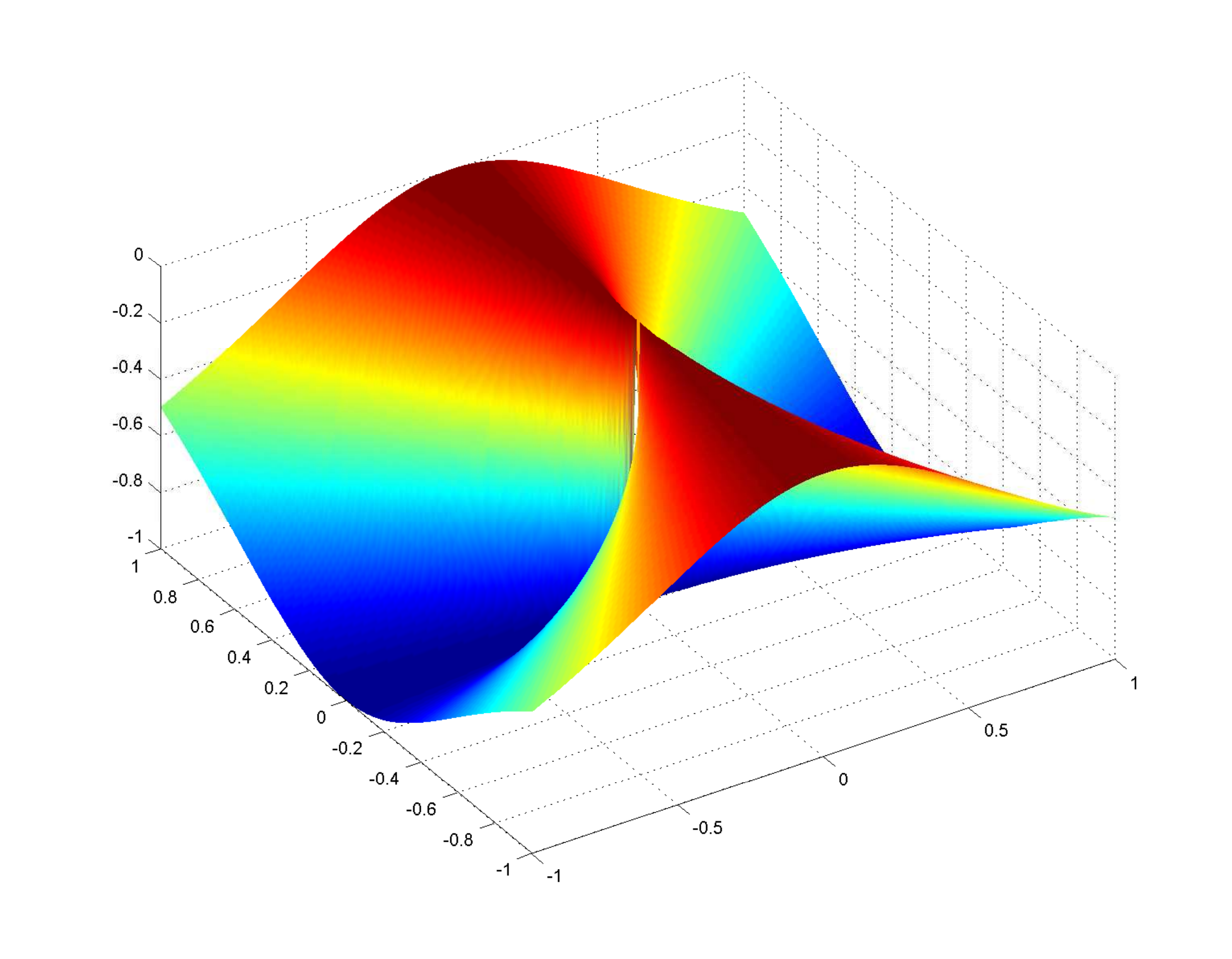}
  \captionof{figure}{The function $(w_1,w_2)\mapsto-\frac{w_1^2}{w_1^2+w_2^2}$, corresponding to $F_A(\bw)$   where $A=(1~0~; 0~0)$. It is invariant to re-scaling of $\bw$, and attains a minimum at $(a,0)$ for any $a\neq 0$.}
  \label{fig:rq}
\end{minipage}%
\hskip 0.1\textwidth
\begin{minipage}{.45\textwidth}
  \centering
  \includegraphics[scale=1.7]{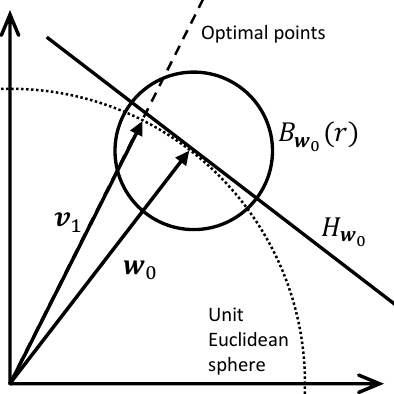}
  \captionof{figure}{Illustration of the construction of the convex set on which $F_A$ is strongly convex and smooth.
  $\bv_1$ is the leading eigenvector of $A$, and a minimum of $F_A$ (as well as any re-scaling of $\bv_1$). $\bw_0$ is a nearby
  unit vector, and we consider the intersection of a hyperplane orthogonal to $\bw_0$, and an Euclidean ball centered at $\bw_0$.}
  \label{fig:const}
\end{minipage}
\end{figure}

\begin{theorem}
  For the matrix $A$ above, the Hessian of $F_A$ is not positive semidefinite for all but a measure-zero set.
\end{theorem}
\begin{proof}
  The leading eigenvector of $A$ is $\bv_1=(1,0)$, and $F_A(\bw)=-\frac{w_1^2}{w_1^2+w_2^2}$. The
  Hessian of this function at some $\bw$ equals
  \[
                           \frac{2}{(w_1^2+w_2^2)^3}\left(\begin{array}{cc}
                             w_2^2(3w_1^2-w_2^2) & -2w_1w_2(w_1^2-w_2^2) \\
                             -2w_1w_2(w_1^2-w_2^2) & w_1^2(w_1^2-3w_2^2) \\
                           \end{array}
                         \right).
  \]
  The determinant of this $2\times 2$ matrix equals
  \begin{align*}
  &\frac{4}{(w_1^2+w_2^2)^6}\left(w_1^2w_2^2(3w_1^2-w_2^2)(w_1^2-3w_2^2)-4w_1^2w_2^2(w_1^2-w_2^2)^2\right)\\
  &=\frac{4w_1^2w_2^2}{(w_1^2+w_2^2)^6}\left((3w_1^2-w_2^2)(w_1^2-3w_2^2)-4(w_1^2-w_2^2)^2\right)\\
  &=\frac{4w_1^2w_2^2}{(w_1^2+w_2^2)^6}\left(-(w_1^2+w_2^2)^2\right)~=~-\frac{4w_1^2w_2^2}{(w_1^2+w_2^2)^4},
  \end{align*}
  which is always non-positive, and strictly negative for $\bw$ for which $w_1w_2\neq 0$ (which holds for all but a measure-zero
  set of $\reals^d$). Since the determinant of a positive semidefinite matrix is always non-negative, this implies that the Hessian isn't
  positive semidefinite
  for any such $\bw$.
\end{proof}

The theorem implies that we indeed cannot use convex optimization tools as-is
on the function $F_A$, even if we're close to an optimum. However, the non-convexity was shown for $F_A$ as a
function over the entire Euclidean space, so the result does not preclude the
possibility of having convexity on a more constrained, lower-dimensional set.
In fact, this is what we are going to do next: We will show that if we are
given some point $\bw_0$ close enough to an optimum, then we can explicitly
construct a simple convex set, such that \begin{itemize} \item The set
includes an optimal point of $F_A$. \item The function $F_A$ is
$\Ocal(1)$-smooth and $\lambda$-strongly convex in that set. \end{itemize}
This means that we can potentially use a two-stage approach: First, we use some existing
algorithm (such as VR-PCA) to find $\bw_0$, and then switch to a convex
optimization algorithm designed to handle functions with a finite sum
structure (such as $F_A$). Since the runtime of such algorithms scale better
than VR-PCA, in terms of the dependence on $\lambda$, we can hope for an
overall runtime improvement.

Unfortunately, this has a catch: To make it work, we need to have $\bw_0$
\emph{very} close to the optimum -- in fact, we require
$\norm{\bv_1-\bw_0}\leq \Ocal(\lambda)$, and we show (in
\thmref{thm:convextight}) that such a dependence on the eigengap $\lambda$
cannot be avoided (perhaps up to a small polynomial factor). The issue is that the runtime to get such a $\bw_0$, using stochastic-based approaches we are aware of, would scale at least quadratically with $1/\lambda$, but getting dependence better than quadratic was our problem to begin with. For example, the runtime guarantee using VR-PCA to get
such a point $\bw_0$ (even if we start from a good point as specified in \thmref{thm:main}) is on the order of
\[
d\left(n+\frac{1}{\lambda^2}\right)\log\left(\frac{1}{\lambda}\right),
\]
whereas the best known guarantees on getting an $\epsilon$-optimal solution
for $\lambda$-strongly convex and smooth functions (see
\cite{agarwal2014lower}) is on the order of
\[
d\left(n+\sqrt{\frac{n}{\lambda}}\right)\log\left(\frac{1}{\epsilon}\right).
\]
Therefore, the total runtime we can hope for would be on the order of
\begin{equation}\label{eq:runtime1}
  d\left(\left(n+\frac{1}{\lambda^2}\right)\log\left(\frac{1}{\lambda}\right)+\left(n+\sqrt{\frac{n}{\lambda}}\right)\log\left(\frac{1}{\epsilon}\right)\right).
\end{equation}
In comparison, the runtime guarantee of using just VR-PCA to get an
$\epsilon$-accurate solution is on the order of
\begin{equation}\label{eq:runtime2}
d\left(n+\frac{1}{\lambda^2}\right)\log\left(\frac{1}{\epsilon}\right).
\end{equation}
Unfortunately, \eqref{eq:runtime2} is the same as \eqref{eq:runtime1} up to
log-factors, and the difference is not significant unless the required
accuracy $\epsilon$ is extremely small (exponentially small in
$n,1/\lambda$). Therefore, our construction is mostly of theoretical
interest. However, it still shows that asymptotically, as
$\epsilon\rightarrow 0$, it is indeed possible to have runtime scaling better
than \eqref{eq:runtime2}. This might hint that designing practical
algorithms, with better runtime guarantees for our problem, may indeed be
possible.

To explain our construction, we need to consider two convex sets: Given a unit
vector $\bw_0$, define the hyperplane tangent to $\bw_0$,
\[
H_{\bw_0} = \{\bw:\inner{\bw,\bw_0}=1\}
\]
as well as a Euclidean ball of radius $r$ centered at $\bw_0$:
\[
B_{\bw_0}(r) = \{\bw:\norm{\bw-\bw_0}\leq r\}
\]
The convex set we use, given such a $\bw_0$, is simply the intersection of
the two, $H_{\bw_0}\cap B_{\bw_0}(r)$, where $r$ is a sufficiently small
number (see Figure \ref{fig:const}).

The following theorem shows that if $\bw_0$ is $\Ocal(\lambda)$-close to
an optimal point (a leading eigenvector $\bv_1$ of $A$), and we choose the radius of
$B_{w_0}(r)$ appropriately, then $H_{\bw_0}\cap B_{\bw_0}(r)$ contains an optimal point, and the function $F_A$
is indeed $\lambda$-strongly convex and smooth on that set. For simplicity,
we will assume that $A$ is scaled to have spectral norm of $1$, but the
result can be easily generalized.

\begin{theorem}\label{thm:convex}
  For any positive semidefinite $A$ with spectral norm $1$, eigengap $\lambda$ and a leading eigenvector $\bv_1$, and any
  unit vector $\bw_0$ such that $\norm{\bw_0-\bv_1}\leq \frac{\lambda}{44}$,
  the function $F_A(\bw)$ is $20$-smooth and
  $\lambda$-strongly convex on the convex set $H_{\bw_0}\cap B_{\bw_0}\left(\frac{\lambda}{22}\right)$, which
  contains a global optimum of $F_A$.
\end{theorem}

The proof of the theorem appears in \subsecref{subsec:proofconvex}. Finally,
we show below that a polynomial dependence on the eigengap $\lambda$ is
unavoidable, in the sense that the convexity property is lost if $\bw_0$ is
significantly further away from $\bv_1$.

\begin{theorem}\label{thm:convextight}
  For any $\lambda,\epsilon\in \left(0,\frac{1}{2}\right)$, there exists a positive semidefinite matrix $A$ with spectral norm $1$, eigengap $\lambda$, and leading eigenvector
  $\bv_1$, as well as a unit vector $\bw_0$ for which $\norm{\bv_1-\bw_0}\leq \sqrt{2(1+\epsilon)\lambda)}$,
  such that $F_A$ is not convex in any neighborhood of $\bw_0$ on $H_{\bw_0}$.
\end{theorem}
\begin{proof}
  Let
  \[
  A=\left(\begin{array}{ccc}
             1 & 0 & 0 \\
             0 & 1-\lambda & 0 \\
             0 & 0 & 0 \\
           \end{array}
         \right),
  \]
  for which $\bv_1=(1,0,0)$, and take
  \[
  \bw_0 = (\sqrt{1-p^2},0,p),
  \]
  where $p =
  \sqrt{(1+\epsilon)\lambda}$ (which ensures
  $\norm{\bv_1-\bw_0}^2 = \sqrt{2p^2} = \sqrt{2(1+\epsilon)\lambda}$). Consider the ray $\{(\sqrt{1-p^2},t,p):t\geq 0\}$, and note that it starts from $\bw_0$ and lies
  in $H_{\bw_0}$. The function $F_A$ along that ray
  (considering it as a function of $t$) is of the form
  \[
  -\frac{(1-p^2)+(1-\lambda)t^2}{(1-p^2)+t^2+p^2}~=~ -\frac{1-p^2+(1-\lambda)t^2}{1+t^2}.
  \]
  The second derivative with respect to $t$ equals
  \[
  -2\frac{(3t^2-1)(\lambda-p^2)}{(t^2+1)^3}~=~
  2\frac{(3t^2-1)\epsilon\lambda}{(t^2+1)^3},
  \]
  where we plugged in the definition of $p$. This is a negative quantity for any $t<\frac{1}{\sqrt{3}}$. Therefore, the function
  $F_A$ is strictly concave (and not convex) along the ray we have defined and close enough to $\bw_0$, and therefore
  isn't convex in any neighborhood of $\bw_0$ on $H_{\bw_0}$.
\end{proof}

\section{Proofs}\label{sec:proofs}

\subsection{Proof of \thmref{thm:main}}\label{subsec:proofmain}

Although the proof structure generally mimics the proof of Theorem 1 in
\cite{shamir2015stochastic} for the $k=1$ special case, it is more intricate
and requires several new technical tools. To streamline the presentation of
the proof, we begin with proving a series of auxiliary lemmas in Subsection
\ref{subsec:proofaux}, and then move to the main proof in Subsection
\ref{subsec:proofmain}. The main proof itself is divided into several steps,
each constituting one or more lemmas.

Throughout the proof, we use the well-known facts that for all matrices
$B,C,D$ of suitable dimensions, $\Tr(B+C)=\Tr(B)+\Tr(C)$, $\Tr(BC)=\Tr(CB)$,
$\Tr(BCD)=\Tr(DBC)$, and $\Tr(B^\top B)=\norm{B}_F^2$. Moreover, since $\Tr$
is a linear operation, $\E[\Tr(B)]=\E[\Tr(B)]$ for a random matrix $B$.

\subsubsection{Auxiliary Lemmas}\label{subsec:proofaux}

\begin{lemma}\label{lem:zc}
  For any $B,C,D\succeq 0$, it holds that $\Tr(BC)\geq \Tr(B(C-D))$ and $\Tr(BC)\geq
  \Tr((B-D)C)$.
\end{lemma}
\begin{proof}
  It is enough to prove that for any positive semidefinite matrices
  $E,G$, it holds that $\Tr(EG)\geq 0$. The lemma follows by taking
  either
  $E=B,G=D$ (in which case, $\Tr(BC)=\Tr(B(C-D))+\Tr(BD)\geq
  \Tr(B(C-D))$), or $E=D,G=C$ (in which case, $\Tr(BC)=\Tr((B-D)C)+\Tr(DC)\geq \Tr((B-D)C)$).

  Any positive semidefinite matrix $M$ can be written as the product $M^{1/2}M^{1/2}$ for some
  symmetric matrix $M^{1/2}$ (known as the matrix square root of $M$). Therefore,
  \begin{align*}
  \Tr(EG)&=\Tr(E^{1/2}E^{1/2}G^{1/2}G^{1/2})=\Tr(G^{1/2}E^{1/2}E^{1/2}G^{1/2})\\
  &=
  \Tr((E^{1/2}G^{1/2})^{\top}(E^{1/2}G^{1/2}))=\norm{E^{1/2}G^{1/2}}_F^2 \geq 0.
  \end{align*}
\end{proof}

\begin{lemma}\label{lem:trinv}
  If $B\succeq 0$ and $C\succ 0$, then
  \[
  \Tr(BC^{-1}) \geq \Tr(B(2I-C)),
  \]
  where $I$ is the identity matrix.
\end{lemma}
\begin{proof}
  We begin by proving the one-dimensional case, where $B,C$ are scalars $b\geq 0,c>0$.
  The inequality then becomes $bc^{-1}\geq b(2-c)$, which is equivalent to
  $1\geq c(2-c)$, or upon rearranging, $(c-1)^2\geq 0$, which trivially
  holds.

  Turning to the general case, we note that by \lemref{lem:zc}, it is enough to prove that
  $C^{-1}-(2I-C)\succeq 0$. To prove this, we make a couple of observations. The
  positive definite matrix $C$ (like any positive definite matrix) has a
  singular value decomposition which can be written as
  $USU^\top$, where $U$ is an orthogonal matrix, and
  $S$ is a diagonal matrix with positive entries. Its inverse is
  $US^{-1}U^\top$, and $2I-C = 2I-USU^\top = U(2I-S)U^\top$. Therefore,
  \[
  C^{-1}-(2I-C) = US^{-1}U^\top-U(2I-S)U^\top = U(S^{-1}-(2I-S))U^\top.
  \]
  To show this matrix is positive semidefinite, it is enough to show that each diagonal entry
  of $S^{-1}-(2I-S)$ is non-negative. But this reduces to the one-dimensional
  result we already proved, when $b=1$ and $c>0$ is any diagonal entry in $S$. Therefore,
  $C^{-1}-(2I-C)\succeq 0$, from which the result follows.
\end{proof}

\begin{lemma}\label{lem:frsp}
  For any matrices $B,C$,
  \[
  \Tr(BC) \leq \norm{B}_F\norm{C}_F
  \]
  and
  \[
  \norm{BC}_F \leq \norm{B}_{sp}\norm{C}_F.
  \]
\end{lemma}
\begin{proof}
  The first inequality is immediate from Cauchy-Shwartz. As to the second
  inequality, letting $\bc_i$ denote the $i$-th column of $C$, and $\norm{\cdot}_2$ the
  Euclidean norm for vectors,
  \[
  \norm{BC}_F = \sqrt{\sum_{i}\norm{B\bc_i}_2^2} \leq \sqrt{\sum_{i}\left(\norm{B}_{sp}\norm{\bc_i}_2\right)^2}
  = \norm{B}_{sp}\sqrt{\sum_i \norm{\bc_i}_2^2} = \norm{B}_{sp}\norm{C}_F.
  \]
\end{proof}

\begin{lemma}\label{lem:taylor}
  Let $B_1,B_2,Z_1,Z_2$ be $k\times k$ square matrices, where $B_1,B_2$
  are fixed and $Z_1,Z_2$ are stochastic and zero-mean (i.e. their
  expectation is the all-zeros matrix). Furthermore, suppose that for some fixed
  $\alpha,\gamma,\delta>0$, it holds with probability $1$ that
  \begin{itemize}
    \item For all $\nu\in [0,1]$, $B_2+\nu Z_2 \succeq \delta I$.
    \item $\max\{\norm{Z_1}_F,\norm{Z_2}_F\}\leq \alpha$.
    \item $\norm{B_1+\eta Z_1}_{sp} \leq \gamma$.
  \end{itemize}
  Then
  \[
  \E\left[\Tr\left((B_1+Z_1)(B_2+Z_2)^{-1}\right)\right]\geq \Tr(B_1 B_2^{-1})-\frac{\alpha^2(1+\gamma/\delta)}{\delta^2}.
  \]
\end{lemma}
\begin{proof}
  Define the function
  \[
  f(\nu)=\Tr\left((B_1+\nu Z_1)(B_2+\nu
  Z_2)^{-1}\right)~,~~~~\nu\in [0,1].
  \]
  Since $B_2+\nu Z_2$ is positive definite,
  it is always invertible, hence $f(\nu)$ is indeed well-defined. Moreover,
  it can be differentiated with respect to $\nu$, and we have
  \[
  f'(\nu) = \Tr\left(Z_1(B_2+\nu Z_2)^{-1}-(B_1+\nu Z_1)(B_2+\nu Z_2)^{-1}Z_2(B_2+\nu Z_2)^{-1}\right).
  \]
  Again differentiating with respect to $\nu$, we have
  \begin{align*}
  f''(\nu) ~&=~ \Tr\Big(-2Z_1(B_2+\nu Z_2)^{-1}Z_2(B_2+\nu Z_2)^{-1}\\
  &~~~~~~~~~~~+2(B_1+\nu Z_1)(B_2+\nu Z_2)^{-1}Z_2(B_2+\nu Z_2)^{-1}Z_2(B_2+\nu Z_2)^{-1}\Big)\\
  &=2\Tr\Big(\Big(-Z_1+(B_1+\nu Z_1)(B_2+\nu Z_2)^{-1}Z_2\Big)(B_2+\nu Z_2)^{-1}Z_2(B_2+\nu Z_2)^{-1}\Big).
  \end{align*}
  Using \lemref{lem:frsp} and the triangle inequality, this is at most
  \begin{align*}
    &2\norm{-Z_1+(B_1+\nu Z_1)(B_2+\nu Z_2)^{-1}Z_2}_F\norm{(B_2+\nu Z_2)^{-1}Z_2(B_2+\nu Z_2)^{-1}}_F\\
    &\leq 2\left(\norm{Z_1}_F+\norm{(B_1+\nu Z_1)(B_2+\nu Z_2)^{-1}Z_2}_F\right)\norm{(B_2+\nu Z_2)^{-1}}_{sp}^2\norm{Z_2}_F\\
    &\leq 2\left(\norm{Z_1}_F+\norm{B_1+\nu Z_1}_{sp}\norm{\left(B_2+\nu Z_2\right)^{-1}}_{sp}\norm{Z_2}_F\right)\norm{(B_2+\nu Z_2)^{-1}}_{sp}^2\norm{Z_2}_F\\
    &\leq 2\left(\alpha+\gamma\frac{1}{\delta}\alpha\right)\frac{1}{\delta^2}\alpha ~=~ \frac{2\alpha^2(1+\gamma/\delta)}{\delta^2}.
  \end{align*}
  Applying a Taylor expansion to $f(\cdot)$ around $\nu=0$, with
  a Lagrangian remainder term, and substituting the values for $f'(\nu),f''(\nu)$,
  we can lower bound $f(1)$ as follows:
  \begin{align*}
    f(1) &\geq f(0)+f'(0)*(1-0)-\frac{1}{2}\max_{\nu}|f''(\nu)|*(1-0)^2\\
    &=\Tr\left(B_1B_2^{-1}\right)+\Tr\left(Z_1B_2^{-1}-B_1B_2^{-1}Z_2B_2^{-1}\right)-\frac{\alpha^2(1+\gamma/\delta)}{\delta^2}.
  \end{align*}
  Taking expectation over $Z_1,Z_2$, and recalling they are zero-mean, we get that
  \[
  \E[f(1)] \geq \Tr\left(B_1B_2^{-1}\right)-\frac{\alpha^2(1+\gamma/\delta)}{\delta^2}.
  \]
  Since $\E[f(1)] = \E\left[\Tr\left((B_1+Z_1)(B_2+Z_2)^{-1}\right)\right]$,
  the result in the lemma follows.
\end{proof}

\begin{lemma}\label{lem:trexp}
  Let $U_1,\ldots,U_k$ and $R_1,R_2$ be positive semidefinite matrices, such that $R_2-R_1\succeq 0$, and
  define the function
  \[
  f(x_1\ldots x_k)=\Tr\left(\left(\sum_{i=1}^{k}x_iU_i+R_1\right)\left(\sum_{i=1}^{k}x_iU_i+R_2\right)^{-1}\right).
  \]
  over all $(x_1\ldots x_k)\in [\alpha,\beta]^d$ for some $\beta\geq \alpha\geq 0$.
  Then $~\min_{(x_1\ldots x_k)\in[\alpha,\beta]^d}f(\bx) =
  f(\alpha,\ldots,\alpha)$.
\end{lemma}
\begin{proof}
  Taking a partial derivative of $f$ with respect to some $x_j$, we have
  \begin{align*}
  &\frac{\partial}{\partial x_j}f(\bx)\\
  &= \Tr\left(U_j\left(\sum_{i=1}^{k}x_iU_i+R_2\right)^{-1}
  -\left(\sum_{i=1}^{k}x_iU_i+R_1\right)\left(\sum_{i=1}^{k}x_iU_i+R_2\right)^{-1}U_j\left(\sum_{i=1}^{k}x_iU_i+R_2\right)^{-1}\right)\\
  &=\Tr\left(\left(I-\left(\sum_{i=1}^{k}U_i+R_1\right)\left(\sum_{i=1}^{k}x_iU_i+R_2\right)^{-1}\right)U_j\left(\sum_{i=1}^{k}x_iU_i+R_2\right)^{-1}\right)\\
  &=\Tr\left(\left(\left(\sum_{i=1}^{k}x_iU_i+R_2\right)-\left(\sum_{i=1}^{k}x_iU_i+R_1\right)\right)\left(\sum_{i=1}^{k}x_iU_i+R_2\right)^{-1}U_j\left(\sum_{i=1}^{k}x_iU_i+R_2\right)^{-1}\right)\\
  &=\Tr\left(\left(R_2-R_1\right)\left(\sum_{i=1}^{k}x_iU_i+R_2\right)^{-1}U_j\left(\sum_{i=1}^{k}x_iU_i+R_2\right)^{-1}\right).
  \end{align*}
  By the lemma's assumptions, each matrix in the product above is positive semidefinite, hence the product
  is positive semidefinite, and the trace is non-negative. Therefore, $\frac{\partial}{\partial x_j}f(\bx)\geq 0$,
  which implies that the function is minimized when each $x_j$ takes its smallest possible value, i.e. $\alpha$.
\end{proof}

\begin{lemma}\label{lem:fratio}
  Let $B$ be a $k\times k$ matrix with minimal singular value $\delta$. Then
  \[
  1-\frac{\norm{B^\top B}_F^2}{\norm{B}_F^2} ~\geq~ \max\left\{1-\norm{B}_F^2~,~\frac{\delta^2}{k}\left(k-\norm{B}_F^2\right)\right\}.
  \]
\end{lemma}
\begin{proof}
  We have
  \[
  1-\frac{\norm{B^\top B}_F^2}{\norm{B}_F^2} \geq 1-\frac{\norm{B}_F^2\norm{B}_F^2}{\norm{B}_F^2} = 1-\norm{B}_F^2,
  \]
  so it remains to prove $1-\frac{\norm{B^\top B}_F^2}{\norm{B}_F^2}\geq
  \frac{\delta^2}{k}\left(k-\norm{B}_F^2\right)$.
  Let $\sigma_1,\ldots,\sigma_k$ denote the vector of singular values of $B$.
  The singular values of $B^\top B$ are $\sigma_1^2,\ldots,\sigma_k^2$, and the Frobenius
  norm of a matrix equals the Euclidean norm of its vector of singular values. Therefore,
  the lemma is equivalent to requiring
  \[
  1-\frac{\sum_{i=1}^{k}\sigma_i^4}{\sum_{i=1}^{k}\sigma_i^2} ~\geq~ \frac{\delta^2}{k}\left(k-\sum_{i=1}^{k}\sigma_i^2\right),
  \]
  assuming $\sigma_i \in [\delta,1]$ for all $i$. This holds since
  \begin{align*}
    1-\frac{\sum_{i}\sigma_i^4}{\sum_{i}\sigma_i^2}
    ~=~ \frac{\sum_{i}\sigma_i^2-\sum_{i}\sigma_i^4}{\sum_{i}\sigma_i^2}
    ~=~ \frac{\sum_{i}\sigma_i^2\left(1-\sigma_i^2\right)}{\sum_{i}\sigma_i^2}
    ~\geq~ \frac{\delta^2\sum_{i}\left(1-\sigma_i^2\right)}{k}
    ~=~ \frac{\delta^2}{k}\left(k-\sum_{i}\sigma_i^2\right).
  \end{align*}
\end{proof}

\begin{lemma}\label{lem:wwt}
    For any $d\times k$ matrices $C,D$ with orthonormal columns, let
    \[
    D_C = \arg\min_{DB~:~\left(DB\right)^\top\left(DB\right)=I}\norm{C-DB}_F^2
    \]
    be the nearest orthonormal-columns matrix to $C$ in the column space of $D$ (where $B$ is a $k\times k$ matrix). Then
    the matrix $B$ minimizing the above equals $B=VU^\top$, where $C^\top D=USV^\top$
    is the SVD decomposition of $C^\top D$, and it holds that
    \[
    \norm{C-D_C}_F^2\leq 2(k-\norm{C^\top D}_F^2).
    \]
\end{lemma}
\begin{proof}
  Since $D$ has orthonormal columns, we have $D^\top D=I$, so the definition of $B$ is equivalent to
    \[
    B = \arg\min_{B~:~B^\top B=I}\norm{C-DB}_F^2.
    \]
  This is the orthogonal Procrustes problem (see e.g.
  \cite{golub2012matrix}), and the solution is easily shown to be $B=VU^\top$ where $USV^\top$ is the SVD decomposition of $C^\top D$.
  In this case, and using the fact that $\norm{C}_F^2=\norm{D}_F^2=k$ (as $C,D$ have orthonormal columns), we
  have that $\norm{C-D_C}_F^2$ equals
  \[
  \norm{C-DB}_F^2 = \norm{C}_F^2+\norm{D}_F^2-2\Tr(C^\top D B) = 2\left(k-\Tr(USV^\top (VU^\top))\right)=2\left(k-\Tr(USU^\top)\right).
  \]
  Since the trace function is similarity-invariant, this equals
  $2k-\Tr(S)$. Let $s_1\ldots,s_k$ be the diagonal elements of $S$, and note that they
  can be at most $1$ (since they are the singular values of $C^\top D$, and both $C$ and $D$ have orthonormal columns).
  Recalling that the Frobenius norm equals the Euclidean norm of the singular
  values, we can therefore
  upper bound the above as follows:
  \[
  2\left(k-\Tr(USU^\top)\right) = 2\left(k-\Tr(S)\right)=2\left(k-\sum_{i=1}^{k}s_i\right) \leq 2\left(k-\sum_{i=1}^{k}s_i^2\right)
  = 2\left(k-\norm{C^\top D}_F^2\right).
  \]
\end{proof}

\begin{lemma}\label{lem:knn}
  Let $W_{t},W'_{t}$ be as defined in Algorithm \ref{alg:algblock}, where we assume
  $\eta < \frac{1}{3}$. Then for any
  $d\times k$ matrix $V_k$ with orthonormal columns, it holds that
  \[
  \left|\norm{V_k^\top W_{t}}_F^2-\norm{V_k^\top W_{t-1}}_F^2\right|\leq \frac{12k\eta}{1-3\eta}.
  \]
\end{lemma}
\begin{proof}
Letting $\bs_t, \bs_{t-1}$ denote the vectors of singular values of $V_k^\top
W_t$ and $V_k^\top W_{t-1}$, and noting that they are both in $[0,1]^k$ (as
$V_k,W_{t-1},W_{t}$ all have orthonormal columns), the left hand side of the
inequality in the lemma statement equals
\[
|\norm{\bs_t}^2-\norm{\bs_{t-1}}^2| = \left(\norm{\bs_t}_2+\norm{\bs_{t-1}}_2\right)
\left|~\norm{\bs_t}_2-\norm{\bs_{t-1}}_2~\right|\leq 2\sqrt{k}\norm{\bs_t-\bs_{t-1}}_2
\leq 2k\norm{\bs_t-\bs_{t-1}}_{\infty},
\]
where $\norm{\cdot}_{\infty}$ is the infinity norm. By Weyl's matrix
perturbation theorem\footnote{Using its version for singular values, which
implies that the singular values of matrices $B$ and $B+E$ are different by
at most $\norm{E}_{sp}$.} \cite{horn2012matrix}, this is upper bounded by
\begin{equation}\label{eq:knn}
2k\norm{V_k^\top W_t-V_k^\top W_{t-1}}_{sp}\leq
2k\norm{V_k}_{sp}\norm{W_t-W_{t-1}}_{sp} \leq 2k\norm{W_t-W_{t-1}}_{sp}.
\end{equation}
Recalling the relationship between $W_t$ and $W_{t-1}$ from Algorithm
\ref{alg:algblock}, we have that
\[
W_t' = W_{t-1}+\eta N,
\]
where
\[
\norm{N}_{sp} \leq \norm{\bx_{i_t}\bx_{i_t}^\top W_{t-1}}_{sp}+\norm{\bx_{i_t}\bx_{i_t}^\top\tilde{W}_{s-1}B_{t-1}}_{sp}
+\norm{\frac{1}{n}\sum_{i=1}^{n}\bx_i\bx_i^\top\tilde{W}_{s-1}B_{t-1}}_{sp}\leq 3,
\]
as $W_{t-1},\tilde{W}_{s-1},B_{t-1}$ all have orthonormal columns, and
$\bx_{i_t}\bx_{i_t}^\top$ and $\frac{1}{n}\sum_{i=1}^{n}\bx_i\bx_i^\top$ have
spectral norm at most $1$. Therefore, $W_{t}'$ equals $W_{t-1}$, up to a
matrix perturbation of spectral norm at most $3\eta$. Again by Weyl's
theorem, this implies that the $k$ non-zero singular values of the $d\times
k$ matrix $W'_{t}$ are different from those of $W_{t-1}$ (which has
orthonormal columns) by at most $3\eta$, and hence all lie in
$[1-3\eta,1+3\eta]$. As a result, the singular values of
$\left(W_{t}^{'\top}W'_{t}\right)^{-1/2}$ all lie in
$\left[\frac{1}{1+3\eta},\frac{1}{1-3\eta}\right]$. Collecting these
observations, we have
\begin{align*}
  &\norm{W_{t}-W_{t-1}}_{sp} = \norm{(W_{t-1}+\eta N)
  \left(W_{t-1}^{'\top}W'_{t-1}\right)^{-1/2}-W_{t-1}}_{sp}\\
  &\leq \norm{W_{t-1}\left(\left(W_{t-1}^{'\top}W'_{t-1}\right)^{-1/2}-I\right)+
  \eta N\left(W_{t-1}^{'\top}W'_{t-1}\right)^{-1/2}}_{sp}\\
  &\leq \norm{\left(W_{t-1}^{'\top}W'_{t-1}\right)^{-1/2}-I}_{sp}+
  \eta\norm{N}_{sp}\norm{\left(W_{t-1}^{'\top}W'_{t-1}\right)^{-1/2}}_{sp}\\
  &\leq \frac{3\eta}{1-3\eta}+\frac{3\eta}{1-3\eta} ~=~ \frac{6\eta}{1-3\eta}.
\end{align*}
Plugging back to \eqref{eq:knn}, the result follows.
\end{proof}


\subsubsection{Main Proof}\label{subsec:proofmainmain}

To simplify the technical derivations, note that the algorithm remains the
same if we divide each $\bx_i$ by $\sqrt{r}$, and multiply $\eta$ by $r$.
Since $\max_i \norm{\bx_i}^2\leq r$, this corresponds to running the
algorithm with step-size $\eta r$ rather than $\eta$, on a re-scaled dataset
of points with squared norm at most $1$, and with an eigengap of $\lambda/r$
instead of $\lambda$. Therefore, we can simply analyze the algorithm assuming
that $\max_i \norm{\bx_i}^2\leq 1$, and in the end plug in $\lambda/r$
instead of $\lambda$, and $\eta r$ instead of $\eta$, to get a result which
holds for data with squared norm at most $r$.

\subsubsection*{Part I: Establishing a Stochastic Recurrence Relation}

We begin by focusing on a single iteration $t$ of the algorithm, and analyze
how $\norm{V_k^\top W_t}_F^2$ (which measures the similarity between the
column spaces of $V_k$ and $W_t$) evolves during that iteration. The key
result we need is \lemref{lem:recur} below, which is specialized for our
algorithm in \lemref{lem:recursealg}.

\begin{lemma}\label{lem:recur}
Let $A$ be a $d\times d$ symmetric matrix with all eigenvalues $s_1\geq
s_2\geq\ldots\geq s_d$ in $[0,1]$, and suppose that $s_{k}-s_{k+1}\geq
\lambda$ for some $\lambda>0$.

Let $N$ be a $d\times k$ zero-mean random matrix such that $\norm{N}_{F}\leq
\sigma^F_N$ and $\norm{N}_{sp}\leq \sigma^{sp}_N$ with probability $1$, and
define
\[
r_N = 46~(\sigma^F_N)^2\left(1+\frac{8}{3}\left(\frac{1}{4}\sigma^{sp}_N+2\right)^2\right)\\
\]

Let $W$ be a $d\times k$ matrix with orthonormal columns, and define
\[
W' = (I+\eta A)W+\eta N~~,~~ W'' = W'(W^{'\top}W')^{-1/2},
\]
for some $\eta\in \left[0,\frac{1}{4\max\{1,\sigma^F_N\}}\right]$.

If $V_k=[\bv_1,\bv_2\ldots,\bv_k]$ is the $d\times k$ matrix of $A$'s first
$k$ eigenvectors, then the following holds:
\begin{itemize}
  \item      $ \E\left[1-\norm{V_k^\top W''}_F^2\right] \leq
      \left(1-\frac{4}{5}\eta\lambda\norm{V_k^\top
      W}_F^2\right)\left(1-\norm{V_k^\top W}_F^2\right)+\eta^2 r_N$
  \item If $\norm{V_k^\top W}_F^2\geq k-\frac{1}{2}$, then
\[
  \E_N\left[k-\norm{V_k^\top W''}_F^2\right]\leq
\left(k-\norm{V_k^\top
W}_F^2\right)\left(1-\frac{1}{10}\eta\lambda\right)+\eta^2r_N.
\]
\end{itemize}
\end{lemma}

\begin{proof}
Using the fact that $\Tr(BCD)=\Tr(CDB)$ for any matrices $B,C,D$, we have
\begin{align}
\E\left[\norm{V_k^\top W''}_F^2\right] &= \E\left[\Tr\left(W^{''\top}V_k V_k^\top W^{''}\right)\right]\notag\\
&= \E\left[\Tr\left(\left(W^{'\top}W'\right)^{-1/2}W^{'\top}V_k V_k^\top W'\left(W^{'\top}W'\right)^{-1/2}\right)\right]\notag\\
&= \E\left[\Tr\left(\left(W^{'\top}V_k V_k^\top W'\right)\left(W^{'\top}W'\right)^{-1}\right)\right]\label{eq:pp0}.
\end{align}
By definition of $W'$, we have
\begin{align*}
  W^{'\top}V_k V_k^\top W'&= \left((I+\eta A)W+\eta N\right)^\top V_k V_k^\top \left((I+\eta A)W+\eta N\right)\\
  &= B_1+Z_1,
\end{align*}
where we define
\begin{align*}
B_1 &= W^\top(I+\eta A)V_k V_k^\top (I+\eta A) W+\eta^2 N^\top V_k V_k^\top N\\
Z_1 &= \eta N^\top V_k V_k^\top (I+\eta A)W + \eta W^\top(I+\eta A)V_kV_k^\top N.
\end{align*}
Also, we have
\begin{align*}
  W^{'\top} W'&= \left((I+\eta A)W+\eta N\right)^\top  \left((I+\eta A)W+\eta N\right)\\
  &= B_2+Z_2,
\end{align*}
where
\begin{align*}
  B_2 &= W^\top(I+\eta A) (I+\eta A) W+\eta^2 N^\top N\\
  Z_2 &= \eta N^\top (I+\eta A)W + \eta W^\top(I+\eta A)N.
\end{align*}
With these definitions, we can rewrite \eqref{eq:pp0} as
$\E\left[\Tr((B_1+Z_1)(B_2+Z_2)^{-1})\right]$. We now wish to remove
$Z_1,Z_2$, by applying \lemref{lem:taylor}. To do so, we check the lemma's
conditions:
\begin{itemize}
  \item \emph{$Z_1,Z_2$ are zero mean}: This holds since they are linear in
      $N$, and $N$ is assumed to be zero-mean.
  \item \emph{$B_2+\nu Z_2 \succeq \frac{3}{8}I$ for all $\nu\in [0,1]$}:
      Recalling the definition of $B_2,Z_2$, and the facts that $A\succeq
      0$, $N^\top N\succeq 0$ (by construction), and $W^\top W=I$, we have
      that $B_2\succeq I$. Moreover, the spectral norm of $Z_2$ is at most
      \[
      2\eta\norm{N^\top (I+\eta A)W}_{sp} \leq 2\eta\norm{N}_{sp}\norm{I+\eta A}_{sp}\norm{W}_{sp}
      \leq 2\eta \sigma^{sp}_N(1+\eta) \leq 2\eta\sigma^{F}_N(1+\eta),
      \]
      which by the assumption on $\eta$ is at most
      $2\frac{1}{4}\left(1+\frac{1}{4}\right)= \frac{5}{8}$. This implies
      that the smallest singular value of $B_2+\nu Z_2$ is at least
      $1-\nu(5/8) \geq 3/8$.
  \item \emph{$\max\{\norm{Z_1}_F,\norm{Z_2}_F\}\leq \frac{5}{2}\eta
      \sigma^F_N$}: By definition of $Z_1,Z_2$, and using
      \lemref{lem:frsp}, the Frobenius norm of these two matrices is at
      most
      \[
      2\eta\norm{N}_F\norm{(I+\eta A)}_{sp}\norm{W}_{sp} \leq 2\eta\sigma^F_N(1+\eta),
      \]
      which by the assumption on $\eta$ is at most
  $2\eta\sigma^F_N\left(1+\frac{1}{4}\right)=\frac{5}{2}\eta\sigma^F_N$.
  \item $\norm{B_1+\eta Z_1}_{sp} \leq
      \left(\frac{1}{4}\sigma^{sp}_N+2\right)^2$: Using the definition of
      $B_1,Z_1$ and the assumption $\eta\leq \frac{1}{4}$,
      \begin{align*}
        \norm{B_1+\eta Z_1}_{sp}&\leq \norm{B_1}_{sp}+\eta\norm{Z_1}_{sp}\\
        &\leq (1+\eta)^2+\eta^2(\sigma^{sp}_N)^2+2\eta\sigma^{sp}_N(1+\eta)\\
        &\leq \left(\frac{5}{4}\right)^2+\frac{1}{16}(\sigma^{sp}_N)^2+\frac{5}{8}\sigma^{sp}_N\\
        &< \left(\frac{1}{4}\sigma^{sp}_N+2\right)^2.
      \end{align*}
\end{itemize}
  Applying \lemref{lem:taylor} and plugging back to \eqref{eq:pp0}, we get
  \begin{align}
    \E\left[\norm{V_k^\top W''}_F^2\right] &\geq \E\left[\Tr((B_1+Z_1)(B_2+Z_2)^{-1})\right]\notag\\
    &\geq \Tr\left(B_1 B_2^{-1}\right)- \frac{400}{9}(\eta\sigma^F_N)^2\left(1+\frac{8}{3}\left(\frac{1}{4}\sigma^{sp}_N+2\right)^2\right)\label{eq:pp1}.
  \end{align}
 We now turn to lower bound $\Tr\left(B_1 B_2^{-1}\right)$, by first re-writing
 $B_1,B_2$ in a different form. For $i=1,\ldots,d$, let
    \[
    U_i = W^\top \bv_i \bv_i^\top W,
    \]
    where $\bv_i$ is the eigenvector of $A$ corresponding to the eigenvalue $s_i$. Note that each $U_i$ is positive semidefinite, and $\sum_{i=1}^{d}U_i =
W^\top W = I$. We have
\begin{align}
B_1 &= W^\top(I+\eta A)V_k V_k^\top (I+\eta A) W+\eta^2 N^\top V_k V_k^\top N\notag\\
    &= W^\top\left((I+\eta A)V_k\right)\left((I+\eta A)V_k\right)^\top W+\eta^2 N^\top V_k V_k^\top N\notag\\
    &= \sum_{i=1}^{k}(1+\eta s_i)^2 W^\top \bv_i \bv_i^\top W+\eta^2 N^\top V_k V_k^\top N\notag\\
    &= \sum_{i=1}^{k}(1+\eta s_i)^2 U_i+\eta^2 N^\top V_k V_k^\top N.\label{eq:b1part}
\end{align}
Similarly,
\begin{align}
  B_2 &= W^\top(I+\eta A) (I+\eta A) W +\eta^2 N^\top N\notag\\
  &= \sum_{i=1}^{d}(1+\eta s_i)^2 W^\top \bv_i \bv_i^\top W+\eta^2 N^\top N\notag\\
  &= \sum_{i=1}^{d}(1+\eta s_i)^2 U_i+\eta^2 N^\top N.\label{eq:b2part}
\end{align}
Plugging \eqref{eq:b1part} and \eqref{eq:b2part} back into \eqref{eq:pp1}, we
get
\begin{align}
\E\left[\norm{V_k^\top W''}_F^2\right]&\geq
\Tr\left(\left(\sum_{i=1}^{k}(1+\eta s_1)^2U_i+\eta^2 N^\top V_k V_k^\top N\right)
\left(\sum_{i=1}^{d}(1+\eta s_i)^2U_i+\eta^2 N^\top N\right)^{-1}\right)\notag\\
&~~~~~- \frac{400}{9}(\eta\sigma^F_N)^2\left(1+\frac{8}{3}\left(\frac{1}{4}\sigma^{sp}_N+2\right)^2\right).
\label{eq:pp2}
\end{align}

Recalling that $s_1\geq s_2\geq\ldots \geq s_k$ and letting $\alpha=(1+\eta
s_k)^2,\beta=(1+\eta s_1)^2$, the trace term can be lower bounded by
\begin{align*}
\min_{x_1,\ldots,x_k\in [\alpha,\beta]}
\Tr\left(\left(\sum_{i=1}^{k}x_iU_i+\eta^2 N^\top V_k V_k^\top
N\right) \left(\sum_{i=1}^{k}x_iU_i+\sum_{i=k+1}^{d}(1+\eta s_i)^2U_i+\eta^2 N^\top
N\right)^{-1}\right).
\end{align*}
Applying \lemref{lem:trexp} (noting that as required by the lemma, $
\sum_{i=k+1}^{d}(1+\eta s_i)^2U_i+\eta^2 N^\top N-\eta^2 N^\top V_k V_k^\top
N = \sum_{i=k+1}^{d}(1+\eta s_i)^2U_i+\eta^2
N^\top\left(I-V_kV_k^\top\right)N\succeq 0$), we can lower bound the above by
\[
\Tr\left(\left((1+\eta s_k)^2\sum_{i=1}^{k}U_i+\eta^2 N^\top V_k V_k^\top
N\right) \left((1+\eta s_k)^2\sum_{i=1}^{k}U_i+\sum_{i=k+1}^{d}(1+\eta s_i)^2U_i+\eta^2 N^\top
N\right)^{-1}\right).
\]
Using \lemref{lem:zc}, this can be lower bounded by
\begin{align*}
&\Tr\left(\left((1+\eta s_k)^2\sum_{i=1}^{k}U_i\right) \left((1+\eta s_k)^2\sum_{i=1}^{k}U_i+\sum_{i=k+1}^{d}(1+\eta s_i)^2U_i+\eta^2 N^\top
N\right)^{-1}\right)\\
&=\Tr\left(\left(\sum_{i=1}^{k}U_i\right) \left(\sum_{i=1}^{k}U_i+\sum_{i=k+1}^{d}\left(\frac{1+\eta s_i}{1+\eta s_k}\right)^2U_i+\left(\frac{\eta}{1+\eta s_k}\right)^2 N^\top
N\right)^{-1}\right)\\
\end{align*}
Applying \lemref{lem:trinv}, this is at least
\[
\Tr\left(\left(\sum_{i=1}^{k}U_i\right) \left(2I-\sum_{i=1}^{k}U_i-\sum_{i=k+1}^{d}\left(\frac{1+\eta s_i}{1+\eta s_k}\right)^2U_i-\left(\frac{\eta}{1+\eta s_k}\right)^2 N^\top
N\right)\right).
\]
Recalling that $I=\sum_{i=1}^{d}U_i=\sum_{i=1}^{k}U_i+\sum_{i=k+1}^{d}U_i$,
this can be simplified to
\begin{equation}\label{eq:pp3}
\Tr\left(\left(\sum_{i=1}^{k}U_i\right) \left(\sum_{i=1}^{k}U_i+\sum_{i=k+1}^{d}\left(2-\left(\frac{1+\eta s_i}{1+\eta s_k}\right)^2\right)U_i-\left(\frac{\eta}{1+\eta s_k}\right)^2 N^\top
N\right)\right).
\end{equation}
Since $U_i\succeq 0$, then using \lemref{lem:trinv}, we can lower bound the
expression above by shrinking each of the $\left(2-\left(\frac{1+\eta
s_i}{1+\eta s_k}\right)^2\right)$ terms. In particular, since $s_i\leq
s_k-\lambda$ for each $i\geq k+1$,
\[
2-\left(\frac{1+\eta s_i}{1+\eta s_k}\right)^2~\geq~
2-\frac{1+\eta s_i}{1+\eta s_k} ~\geq~
2-\frac{1+\eta(s_k-\lambda)}{1+\eta s_k}
~=~ 1+\frac{\eta\lambda}{1+\eta s_k},
\]
which by the assumption that $\eta\leq 1/4$ and $s_k\leq s_1\leq 1$, is at
least $1+\frac{4}{5}\eta \lambda$. Plugging this back into \eqref{eq:pp3},
and recalling that $\sum_{i=1}^{d}U_i=I$, we get the lower bound
\begin{align*}
&\Tr\left(\left(\sum_{i=1}^{k}U_i\right) \left(\sum_{i=1}^{k}U_i+\sum_{i=k+1}^{d}\left(1+\frac{4}{5}\eta\lambda\right)U_i-\left(\frac{\eta}{1+\eta s_k}\right)^2 N^\top
N\right)\right)\\
&= \Tr\left(\left(\sum_{i=1}^{k}U_i\right) \left(I+\frac{4}{5}\eta\lambda\left(I-\sum_{i=1}^{k}U_i\right)-\left(\frac{\eta}{1+\eta s_k}\right)^2 N^\top
N\right)\right).
\end{align*}
Again using \lemref{lem:zc}, this is at least
\begin{align*}
   &\Tr\left(\left(\sum_{i=1}^{k}U_i\right) \left(I+\frac{4}{5}\eta\lambda\left(I-\sum_{i=1}^{k}U_i\right)\right)\right)
   -\left(\frac{\eta}{1+\eta s_k}\right)^2 \Tr\left(\left(\sum_{i=1}^{k}U_i\right)N^\top N\right)\\
   &\geq\Tr\left(\left(\sum_{i=1}^{k}U_i\right) \left(I+\frac{4}{5}\eta\lambda\left(I-\sum_{i=1}^{k}U_i\right)\right)\right)
   -\left(\frac{\eta}{1+\eta s_k}\right)^2 \Tr\left(N^\top N\right)\\
   &\geq\Tr\left(\left(\sum_{i=1}^{k}U_i\right) \left(I+\frac{4}{5}\eta\lambda\left(I-\sum_{i=1}^{k}U_i\right)\right)\right)
   -\eta^2\left(\sigma^F_N\right)^2.
\end{align*}
Recall that this is a lower bound on the trace term in \eqref{eq:pp2}.
Plugging it back and slightly simplifying, we get
\[
\E\left[\norm{V_k^\top W''}_F^2\right]\geq
\Tr\left(\left(\sum_{i=1}^{k}U_i\right) \left(I+\frac{4}{5}\eta\lambda\left(I-\sum_{i=1}^{k}U_i\right)\right)\right)-\eta^2 r_N,
\]
where
\[
r_N = 46~(\sigma^F_N)^2\left(1+\frac{8}{3}\left(\frac{1}{4}\sigma^{sp}_N+2\right)^2\right).
\]
The trace term above can be re-written (using the definition of $U_i$ and the
fact that $\Tr(B^\top B)=\norm{B}_F^2$) as
\begin{align*}
  &\Tr\left(\left(W^\top\sum_{i=1}^{k}\bv_i\bv_i^\top W\right) \left(I+\frac{4}{5}\eta\lambda\left(I-W^\top\sum_{i=1}^{k}\bv_i\bv_i^\top W\right)\right)\right)\\
  &=\left(1+\frac{4}{5}\eta\lambda\right)\Tr\left(W^\top V_k V_k^\top W\right)-
  \frac{4}{5}\eta\lambda\Tr\left(\left(W^\top V_k V_k^\top W\right)\left(W^\top V_k V_k^\top W\right)\right)\\
  &= \left(1+\frac{4}{5}\eta\lambda\right)\norm{V_k^\top W}_F^2-
  \frac{4}{5}\eta\lambda\norm{W^\top V_k V_k^\top W}_F^2\\
  &= \norm{V_k^\top W}_F^2\left(1+\frac{4}{5}\eta\lambda\left(1-\frac{\norm{W^\top V_k V_k^\top W}_F^2}{\norm{V_k^\top W}_F^2}\right)\right).
\end{align*}
Applying \lemref{lem:fratio}, and letting $\delta$ denote the minimal
singular value of $V_k^\top W$, this is lower bounded by
\[
\norm{V_k^\top
W}_F^2\left(1+\frac{4}{5}\eta\lambda\max\left\{1-\norm{V_k^\top W}_F^2~,~\frac{\delta^2}{k}\left(k-\norm{V_k^\top
W}_F^2\right)\right\}\right).
\]
Overall, we get that
\begin{equation}\label{eq:ppi}
\E\left[\norm{V_k^\top W''}_F^2\right] \geq \norm{V_k^\top
W}_F^2\left(1+\frac{4}{5}\eta\lambda\max\left\{1-\norm{V_k^\top W}_F^2~,~\frac{\delta^2}{k}\left(k-\norm{V_k^\top
W}_F^2\right)\right\}\right)-\eta^2 r_N.
\end{equation}
We now consider two options:
\begin{itemize}
  \item Taking the first argument of the max term in \eqref{eq:ppi}, we get
      \[
      \E\left[\norm{V_k^\top W''}_F^2\right] \geq \norm{V_k^\top W}_F^2\left(1+\frac{4}{5}\eta\lambda\left(1-\norm{V_k^\top W}_F^2\right)\right)-\eta^2 r_N.
      \]
      Subtracting $1$ from both sides and simplifying, we get
      \[
      \E\left[1-\norm{V_k^\top W''}_F^2\right] \leq \left(1-\frac{4}{5}\eta\lambda\norm{V_k^\top W}_F^2\right)\left(1-\norm{V_k^\top W}_F^2\right)+\eta^2 r_N.
      \]
  \item Suppose that $\norm{V_k^\top W}_F^2\geq k-\frac{1}{2}$. Taking the
      second argument of the max term in \eqref{eq:ppi}, we get
    \[
    \E\left[\norm{V_k^\top W''}_F^2\right] \geq \norm{V_k^\top
W}_F^2\left(1+\frac{4\eta\lambda\delta^2}{5k}\left(k-\norm{V_k^\top
W}_F^2\right)\right)-\eta^2 r_N.
    \]
    Subtracting both sides from $k$, , we get
\begin{align*}
  \E\left[k-\norm{V_k^\top W''}_F^2\right]&\leq
  \left(k-\norm{V_k^\top W}_F^2\right)-\frac{4\eta\lambda\delta^2}{5k}\norm{V_k^\top W}_F^2\left(k-\norm{V_k^\top W}_F^2\right)+\eta^2r_N\\
  &= \left(k-\norm{V_k^\top W}_F^2\right)\left(1-\frac{4\eta\lambda\delta^2}{5k}\norm{V_k^\top W}_F^2\right)+\eta^2r_N\\
  &\leq \left(k-\norm{V_k^\top W}_F^2\right)\left(1-\frac{4\eta\lambda\delta^2}{5k}\left(k-\frac{1}{2}\right)\right)+\eta^2r_N\\
\end{align*}
Since $k\geq 1$, we can lower bound the $\left(k-\frac{1}{2}\right)$ term
by $\frac{k}{2}$. Moreover, the condition $k-\norm{V_k^\top W}_F^2\leq
\frac{1}{2}$ implies that the singular values $\sigma_1,\ldots,\sigma_k$ of
$V_k^\top W$ satisfy $k-\sum_{i=1}^{k}\sigma_i^2\leq \frac{1}{2}$. But each
$\sigma_i$ is in $[0,1]$ (as $V_k,W$ have orthonormal columns), so no
$\sigma_i$ can be less than $\frac{1}{2}$. This implies that $\delta\geq
\frac{1}{2}$. Plugging the lower bounds $k-\frac{1}{2}\geq \frac{k}{2}$ and
$\delta\geq \frac{1}{2}$ into the above, we get
\[
  \E\left[k-\norm{V_k^\top W''}_F^2\right]\leq
\left(k-\norm{V_k^\top
W}_F^2\right)\left(1-\frac{1}{10}\eta\lambda\right)+\eta^2r_N.
\]
\end{itemize}
\end{proof}

\begin{lemma}\label{lem:recursealg}
    Let $A,W_t$ be as defined in Algorithm \ref{alg:algblock}, and suppose that
    $\eta \in \left[0,\frac{1}{23\sqrt{k}}\right]$. Then the following holds for some positive numerical
    constants
    $c_1,c_2,c_3$:
    \begin{itemize}
        \item $ \E\left[1-\norm{V_k^\top W''}_F^2\right] \leq
            \left(1-c_1\eta\lambda\norm{V_k^\top
            W}_F^2\right)\left(1-\norm{V_k^\top W}_F^2\right)+c_2k\eta^2$
      \item If $\norm{V_k^\top W_t}_F^2\geq k-\frac{1}{2}$, then
        \[
      \E\left[k-\norm{V_k^\top W_{t+1}}_F^2\right]~\leq~
        \left(k-\norm{V_k^\top
W_t}_F^2\right)\left(1-c_1\eta\left(\lambda-c_2\eta\right)\right)\\
~+~c_3\eta^2(k-\norm{V_k^\top \tilde{W}_{s-1}}_F^2).
        \]
    \end{itemize}
  In the above, the expectation is over the random draw of the index $i_t$,
  conditioned on $W_t$ and $\tilde{W}_{s-1}$.
\end{lemma}
\begin{proof}
To apply \lemref{lem:recur}, we need to compute upper bounds $\sigma^F_N$ and
$\sigma^{sp}_N$ on the Frobenius and spectral norms of $N$, which in our case
equals $(\bx_{i_t}\bx_{i_t}^\top-A)(W_t-\tilde{W}_{s-1}B_t)$. Since
$\norm{A}_{sp},\norm{\bx_{i_t}\bx_{i_t}^\top}_{sp}\leq 1$, and
$W_t,\tilde{W}_{s-1},B_t$ have orthonormal columns, the spectral norm of $N$
is at most
\[
\norm{(\bx_{i_t}\bx_{i_t}^\top-A)(W_t-\tilde{W}_{s-1}B_t)}_{sp}\leq
\left(\norm{\bx_{i_t}\bx_{i_t}^\top}_{sp}+\norm{A}_{sp}\right)\left(\norm{W_t}_{sp}+\norm{\tilde{W}_{s-1}}_{sp}\norm{B_t}_{sp}\right)
\leq 4,
\]
so we may take $\sigma^{sp}_N=4$. As to the Frobenius norm, using
\lemref{lem:frsp} and a similar calculation, we have
\[
\norm{N}_{F}^2\leq 4\norm{W_t-\tilde{W}_{s-1}B_{t}}_F^2.
\]
To upper bound this, define
\[
V_{W_t} = \arg\min_{V_k B: (V_k B)^\top(V_k B)=I}\norm{W_t-V_k B}_F^2
\]
to be the nearest orthonormal-columns matrix to $W_t$ in the column space of
$V_k$, and
\[
\tilde{W}_{V} = \arg\min_{\tilde{W}_{s-1} B: (\tilde{W}_{s-1} B)^\top(\tilde{W}_{s-1} B)=I}\norm{V_{W_t}-\tilde{W}_{s-1} B}_F^2
\]
to be the nearest orthonormal-columns matrix to $V_{W_t}$ in the column space
of $\tilde{W}_{s-1}$. Also, recall that by definition,
\[
\tilde{W}_{s-1}B_{t} = \arg\min_{\tilde{W}_{s-1}B:(\tilde{W}_{s-1} B)^\top(\tilde{W}_{s-1} B)=I}\norm{W_t-\tilde{W}_{s-1} B}_F^2
\]
is the nearest orthonormal-columns matrix to $W_t$ in the column space of
$\tilde{W}_{s-1}$. Therefore, we must have
$\norm{W_t-\tilde{W}_{s-1}B_{t}}_F^2\leq \norm{W_t-\tilde{W}_{V}}_F^2$. Using
this and \lemref{lem:wwt}, we have
\begin{align*}
  \norm{W_t-\tilde{W}_{s-1}B_{t}}_F^2 &\leq \norm{W_t-\tilde{W}_{V}}_F^2\\
  &=\norm{(W_t-V_{W_t})-(\tilde{W}_{V}-V_{W_t})}_F^2\\
  &\leq 2\norm{W_t-V_{W_t}}_F^2+2\norm{\tilde{W}_{V}-V_{W_t}}_F^2\\
  &= 4\left(k-\norm{V_{k}^\top W_t}_F^2\right)+4\left(k-\norm{V_{W_t}^\top \tilde{W}_{s-1}}_F^2\right).
\end{align*}
By definition of $V_{W_t}$, we have $V_{W_t}=V_k B$ where $B^\top B=B^\top
V_k^\top V_k B = (V_k B)^\top (V_k B)=I$. Therefore $B$ is an orthogonal
$k\times k$ matrix, and $\norm{V_{W_t}^\top \tilde{W}_{s-1}}_F^2=\norm{B^\top
V_{k}^\top \tilde{W}_{s-1}}_F^2=\norm{V_k^\top \tilde{W}_{s-1}}_F^2$, so the
above equals $4(k-\norm{V_k^\top W_t}_F^2)+4(k-\norm{V_k^\top
\tilde{W}_{s-1}}_F^2)$. Overall, we get that the squared Frobenius norm of
$N$ can be upper bounded by
\[
(\sigma^F_N)^2 = 16\left((k-\norm{V_k^\top W_t}_F^2)+(k-\norm{V_k^\top \tilde{W}_{s-1}}_F^2)\right).
\]
Plugging $\sigma^{sp}_N$ and $(\sigma^F_N)^2$ into the $r_N$ as defined in
\lemref{lem:recur}, and picking any $\eta\in [0,\frac{1}{23\sqrt{k}}]$ (which
satisfies the condition in \lemref{lem:recur} that $\eta\in
\left[0,\frac{1}{4\max\{1,\sigma^F_N\}}\right]$, since
$4\max\{1,\sigma^F_n\}\leq 4\max\{1,\sqrt{16*2k}\}< 23\sqrt{k}$), we get
\begin{align*}
r_N &= 736\left((k-\norm{V_k^\top W_t}_F^2)+(k-\norm{V_k^\top \tilde{W}_{s-1}}_F^2)\right)\left(1+\frac{8}{3}\left(\frac{1}{4}4+2\right)^2\right)\\
&\leq
18400~\left((k-\norm{V_k^\top W_t}_F^2)+(k-\norm{V_k^\top \tilde{W}_{s-1}}_F^2)\right).
\end{align*}
This implies that $r_N\leq 36800 k$ always, which by application of
\lemref{lem:recur}, gives the first part of our lemma. As to the second part,
assuming $\norm{V_k^\top W_t}_F^2\geq k-\frac{1}{2}$ and applying
\lemref{lem:recur}, we get that
        \begin{align*}
      \E\left[k-\norm{V_k^\top W_{t+1}}_F^2\right]&\leq
        \left(k-\norm{V_k^\top
W_t}_F^2\right)\left(1-\frac{1}{10}\eta\lambda\right)\\&~~~~~~~+18400~\eta^2\left((k-\norm{V_k^\top W_t}_F^2)+(k-\norm{V_k^\top \tilde{W}_{s-1}}_F^2)\right)\\
    &=\left(k-\norm{V_k^\top
W_t}_F^2\right)\left(1-\eta\left(\frac{1}{10}\lambda-18400\eta\right)\right)\\&~~~~~~~+18400~\eta^2(k-\norm{V_k^\top \tilde{W}_{s-1}}_F^2).
        \end{align*}
This corresponds to the lemma statement.
\end{proof}

\subsubsection*{Part II: Solving the Recurrence Relation for a Single Epoch}

Since we focus on a single epoch, we drop the subscript from
$\tilde{W}_{s-1}$ and denote it simply as $\tilde{W}$.

Suppose that $\eta=\alpha\lambda$, where $\alpha$ is a sufficiently small
constant to be chosen later. Also, let
\[
b_t = k-\norm{V_k^\top W_t}_F^2~~~\text{and}~~~ \tilde{b} = k-\norm{V_k^\top\tilde{W}}_F^2.
\]
Then \lemref{lem:recursealg} tells us that if $\alpha$ is a sufficiently
small constant, $b_t\leq \frac{1}{2}$, then
\begin{equation}\label{eq:bform}
\E\left[b_{t+1}\middle| W_t\right] ~\leq~
\left(1-c\alpha\lambda^2\right)b_t
+c'\alpha^2\lambda^2\tilde{b}
\end{equation}
for some numerical constants $c,c'$.

\begin{lemma}\label{lem:recurse}
Let $B$ be the event that $b_t\leq \frac{1}{2}$ for all $t=0,1,2,\ldots,m$.
Then for certain positive numerical constants $c_1,c_2,c_3$, if $\alpha\leq
c_1$, then
\[
\E[b_{m}|B]\leq \left(\left(1-c_2\alpha\lambda^2\right)^{m}+c_3\alpha\right) \tilde{b},
\]
where the expectation is over the randomness in the current epoch.
\end{lemma}
\begin{proof}
  Recall that $b_t$ is a deterministic function of the random variable
  $W_t$, which depends in turn on $W_{t-1}$ and the random instance chosen at round
  $t$. We assume that $W_0$ (and hence $\tilde{b}$) are fixed, and consider how $b_t$ evolves as a function of $t$. Using \eqref{eq:bform}, we have
  \begin{align*}
  \E[b_{t+1}|W_{t},B] = \E\left[b_{t+1}|W_t,b_{t+1}\leq \frac{1}{2}\right]
  ~\leq~ \E[b_{t+1}|W_t] ~\leq~ \left(1-c\alpha\lambda^2\right)b_t
~+~c'\alpha^2\lambda^2\tilde{b}.
\end{align*}
Note that the first equality holds, since conditioned on $W_t$, $b_{t+1}$ is
independent of $b_1,\ldots,b_{t}$, so the event $B$ is equivalent to just
requiring $b_{t+1}\leq 1/2$.

Taking expectation over $W_t$ (conditioned on $B$), we get that
\begin{align*}
  \E[b_{t+1}|B] ~&\leq~ \E\left[\left(1-c\alpha\lambda^2\right)b_{t}
+c'\alpha^2\lambda^2\tilde{b}\middle| B\right]\\
&=\left(1-c\alpha\lambda^2\right)\E\left[b_{t}|B\right]
+c'\alpha^2\lambda^2\tilde{b}.
\end{align*}
Unwinding the recursion, and using that $b_0=\tilde{b}$, we therefore get
that
\begin{align*}
\E[b_{m}|B]~&\leq~\left(1-c\alpha\lambda^2\right)^{m}\tilde{b}+c'\alpha^2\lambda^2\tilde{b}\sum_{i=0}^{m-1}\left(1-c\alpha\lambda^2\right)^i\\
&\leq~\left(1-c\alpha\lambda^2\right)^{m}\tilde{b}+c'\alpha^2\lambda^2\tilde{b}\sum_{i=0}^{\infty}\left(1-c\alpha\lambda^2\right)^i\\
&=~\left(1-c\alpha\lambda^2\right)^{m}\tilde{b}+c'\alpha^2\lambda^2\tilde{b}\frac{1}{c\alpha\lambda^2}\\
&=~\left(\left(1-c\alpha\lambda^2\right)^{m}+\frac{c'}{c}\alpha\right) \tilde{b}.\\
\end{align*}
as required.
\end{proof}

We now turn to prove that the event $B$ assumed in \lemref{lem:recurse}
indeed holds with high probability:
\begin{lemma}\label{lem:event}
  The following holds for certain positive numerical constants $c_1,c_2,c_3$:
  If $\alpha\leq c_1$, then for any $\beta\in (0,1)$ and $m$, if
  \begin{equation}\label{eq:event}
  \tilde{b}+c_2km\alpha^2\lambda^2 +c_3k\sqrt{m\alpha^2\lambda^2\log(1/\beta)}\leq \frac{1}{2},
  \end{equation}
  then it holds with probability at least $1-\beta$ that
  \[
  b_t~\leq~ \tilde{b}+c_2km\alpha^2\lambda^2 +c_3k\sqrt{m\alpha^2\lambda^2\log(1/\beta)}~\leq~ \frac{1}{2}
  \]
  for all $t=0,1,2,\ldots,m$.
\end{lemma}
\begin{proof}
 To prove the lemma, we analyze the stochastic process $b_0(=\tilde{b}),b_1,b_2,\ldots,b_m$, and
 use a concentration of measure argument. First, we collect the following
 facts:
 \begin{itemize}
     \item \emph{$\tilde{b}=b_0\leq \frac{1}{2}$}: This directly follows
         from the assumption stated in the lemma.
     \item \emph{As long as $b_t\leq \frac{1}{2}$, $\E\left[b_{t+1}\middle|
         W_t\right]\leq b_t+ c_2\alpha^2\lambda^2\tilde{b}$ for some
         constant $c_2$}: Supposing $\alpha$ is sufficiently small, then by
         \eqref{eq:bform},
 \begin{align*}
 \E\left[b_{t+1}\middle| W_t\right] ~&\leq~
\left(1-c\alpha\lambda^2\right)b_t
+c'\alpha^2\lambda^2\tilde{b}
~\leq~ b_t+c'\alpha^2\lambda^2 \tilde{b}.
 \end{align*}
    \item \emph{$|b_{t+1}-b_t|$ is bounded by $c'_3k\alpha\lambda$ for some
        constant $c'_3$}: Applying \lemref{lem:knn}, and assuming that
        $\alpha$ is at most some sufficiently small constant $c_1$ (e.g.
        $\alpha\leq \frac{1}{12}$, so $\eta=\alpha\lambda\leq
        \frac{1}{12}$),
\[
 |b_{t+1}-b_{t}|~=~
 \left|\norm{V_k^\top W_{t+1}}_F^2-\norm{V_k^\top W_{t}}_F^2\right|
 ~\leq~ \frac{12k\eta}{1-3\eta} \leq \frac{12k\alpha \lambda}{3/4} ~=~ 16k\alpha\lambda.
\]
 \end{itemize}
 Armed with these facts, and using the maximal version of the Hoeffding-Azuma inequality \cite{hoeffding1963probability}, it follows that with probability at least
 $1-\beta$, it holds simultaneously for all $t=1,\ldots,m$ (and for $t=0$ by assumption) that
 \[
 b_t\leq \tilde{b}+c_2 m\alpha^2\lambda^2 \tilde{b}+c_3k\sqrt{m\alpha^2\lambda^2\log(1/\beta)}
 \]
 for some constants $c_2,c_3$, as long as the expression above is less than
 $\frac{1}{2}$. If the expression is indeed less than $\frac{1}{2}$, then we
 get that $b_t\leq \frac{1}{2}$ for all $t$. Upper bounding $\tilde{b}$ by $k$ and slightly simplifying, we get the
 statement in the lemma.
\end{proof}

Combining \lemref{lem:recurse} and \lemref{lem:event}, and using Markov's
inequality, we get the following corollary:

\begin{lemma}\label{lem:combine}
Let confidence parameters $\beta,\gamma\in(0,1)$ be fixed. Suppose that
$m,\alpha$ are chosen such that $\alpha\leq c_1$ and
\[
\tilde{b}+c_2km\alpha^2\lambda^2+c_3k\sqrt{m\alpha^2\lambda^2\log(1/\beta)}\leq \frac{1}{2},
\]
where $c_1,c_2,c_3$ are certain positive numerical constants. Then with
probability at least $1-(\beta+\gamma)$, it holds that
\[
b_m \leq \frac{1}{\gamma}
\left(\left(1-c\alpha\lambda^2\right)^{m}+c'\alpha\right) \tilde{b}.
\]
for some positive numerical constants $c,c'$.
\end{lemma}%

\subsubsection*{Part III: Analyzing the Entire Algorithm's Run}

Given the analysis in \lemref{lem:combine} for a single epoch, we are now
ready to prove our theorem. Let
\[
\tilde{b}_s = k-\norm{V_k^\top\tilde{W}_s}_F^2.
\]
By assumption, at the beginning of the first epoch, we have
$\tilde{b}_0=k-\norm{V_k^\top\tilde{W}_0}_F^2\leq \frac{1}{2}$. Therefore, by
\lemref{lem:combine}, for any $\beta,\gamma\in\left(0,\frac{1}{2}\right)$, if
we pick any
\begin{equation}\label{eq:condme}
\alpha\leq \min
\left\{c_1,\frac{1}{2c'}\gamma^2\right\}~~~~\text{and}~~~~
m\geq \frac{3\log(1/\gamma)}{c\alpha\lambda^2}
~~~~\text{such that}~~~~\frac{1}{2}+c_2km\alpha^2\lambda^2+c_3k\sqrt{m\alpha^2\lambda^2\log(1/\beta)}\leq \frac{1}{2},
\end{equation}
then we get with probability at least $1-(\beta+\gamma)$ that
\[
b_m~\leq~
\frac{1}{\gamma}\left(\left(1-c\alpha\lambda^2\right)^{\frac{3\log(1/\gamma)}{c\alpha\lambda^2}}
+\frac{1}{2}\gamma^2\right)\tilde{b}_0
\]
Using the inequality $(1-(1/x))^{ax}\leq \exp(-a)$, which holds for any $x>1$
and any $a$, and taking $x = 1/(c\alpha\lambda^2)$ and $a = 3\log(1/\gamma)$,
we can upper bound the above by
\begin{align*}
&\frac{1}{\gamma}\left(\exp\left(-3\log\left(\frac{1}{\gamma}\right)\right)+\frac{1}{2}\gamma^2\right)\tilde{b}_0\\
&=~ \frac{1}{\gamma}\left(\gamma^3+\frac{1}{2}\gamma^2\right)\tilde{b}_0 ~\leq~ \gamma\tilde{b}_0.
\end{align*}
Since $b_m$ equals the starting point $\tilde{b}_1$ for the next epoch, we
get that $\tilde{b}_1\leq \gamma\tilde{b}_0\leq \gamma\frac{1}{2}$. Again
applying \lemref{lem:combine}, and performing the same calculation we have
that with probability at least $1-(\beta+\gamma)$ over the next epoch,
$\tilde{b}_2~\leq~ \gamma \tilde{b}_1~\leq~ \gamma^2\tilde{b}_0$. Repeatedly
applying \lemref{lem:combine} and using a union bound, we get that after $T$
epochs, with probability at least $1-T(\beta+\gamma)$,
\[
k-\norm{V_k^\top\tilde{W}_T}_F^2~=~\tilde{b}_T ~\leq~ \gamma^T\tilde{b}_0 ~<~ \gamma^T.
\]
Therefore, for any desired accuracy parameter $\epsilon$, we simply need to
use $T=\left\lceil\frac{\log(1/\epsilon)}{\log(1/\gamma)}\right\rceil$
epochs, and get $k-\norm{V_k^\top\tilde{W}_s}_F^2\leq \epsilon$ with
probability at least
$1-T(\beta+\gamma)=1-\left\lceil\frac{\log(1/\epsilon)}{\log(1/\gamma)}\right\rceil(\beta+\gamma)$.

Using a confidence parameter $\delta$, we pick
$\beta=\gamma=\frac{\delta}{2}$, which ensures that the accuracy bound above
holds with probability at least
\[
1-\left\lceil\frac{\log(1/\epsilon)}{\log(2/\delta)}\right\rceil\delta
~\geq~
1-\left\lceil\frac{\log(1/\epsilon)}{\log(2)}\right\rceil\delta
~=~
1-\left\lceil\log_2\left(\frac{1}{\epsilon}\right)\right\rceil\delta.
\]
Substituting this choice of $\beta,\gamma$ into \eqref{eq:condme}, and
recalling that the step size $\eta$ equals $\alpha\lambda$, we get that
$k-\norm{V_k^\top\tilde{W}_T}_F^2\leq \epsilon$ with probability at least
$1-\lceil\log_2(1/\epsilon)\rceil\delta$, provided that
\[
\eta \leq c\delta^2\lambda~~~~,~~~~
m\geq \frac{c'\log(2/\delta)}{\eta \lambda}
~~~~,~~~~km\eta^2+k\sqrt{m\eta^2\log(2/\delta)}\leq c''
\]
for suitable positive constants $c,c',c''$.

To get the theorem statement, recall that the analysis we performed pertains
to data whose squared norm is bounded by $1$. By the reduction discussed at
the beginning of the proof,
 we can apply it to data with squared norm at most $r$, by replacing $\lambda$ with $\lambda/r$,
 and $\eta$ with $\eta r$, leading to the condition
\[
\eta \leq \frac{c\delta^2}{r^2}\lambda~~~~,~~~~
m\geq \frac{c'\log(2/\delta)}{\eta \lambda}
~~~~,~~~~km\eta^2r^2+rk\sqrt{m\eta^2\log(2/\delta)}\leq c''
\]
and establishing the theorem.

\subsection{Proof of \thmref{thm:burn}}\label{subsec:proofburn}

The proof relies mainly on the techniques and lemmas of
\secref{subsec:proofmain}, used to prove \thmref{thm:main}. As done in
\secref{subsec:proofmain}, we will assume without loss of generality that
$r=\max_i \norm{\bx_i}^2$ is at most $1$, and then transform the bound to a
bound for general $r$ (see the discussion at the beginning of
\subsecref{subsec:proofmainmain})

First, we extract the following result, which is essentially the first part
of \lemref{lem:recursealg} (for $k=1$):
\begin{lemma}\label{lem:recurstext}
    Let $A,\bw_t$ be as defined in Algorithm \ref{alg:algvec}, and suppose that
    $\eta \in \left[0,\frac{1}{23}\right]$. Then
    \[
      \E_{i_t}\left[1-\inner{\bv_1,\bw_{t+1}}^2\middle| \bw_t,\tilde{\bw}_{s-1}\right] \leq \left(1-c\eta\lambda\inner{\bv_1,\bw_t}^2\right)
      \left(1-\inner{\bv_1,\bw_{t}}^2\right)+c'\eta^2,
    \]
    for some positive numerical constants $c,c'$.
\end{lemma}
Note that this bound holds regardless of what is $\tilde{\bw}_{s-1}$, and in
particular holds across different epochs of Algorithm \ref{alg:algvec}.
Therefore, it is enough to show that starting from some initial point
$\bw_0$, after sufficiently many stochastic updates as specified in line 6-10
of the algorithm (or in terms of the analysis, sufficiently many applications
of \lemref{lem:recurstext}), we end up with a point $\bw_T$ for which
$1-\inner{\bv_1,\bw_T}\leq \frac{1}{2}$, as required. Note that to simplify
the notation, we will use here a single running index
$\bw_0,\bw_1,\bw_2,\ldots,\bw_T$ (whereas in the algorithm we restarted the
indexing after every epoch).

The proof is based on martingale arguments, quite similar to the ones in
\subsecref{subsec:proofmainmain} but with slight changes. First, we let
\[
b_t = 1-\inner{\bv_1,\bw_t}^2
\]
to simplify notation. We note that $b_0=1-\inner{\bv_1,\bw_0}^2$ is assumed
fixed, whereas $b_1,b_2,\ldots$ are random variables based on the sampling
process. \lemref{lem:recursealg} tells us that if $\eta$ is sufficiently
small, and $b_t \leq 1-\xi$ for some $\xi\in (0,1)$, then
\begin{equation}\label{eq:bformburn}
\E\left[b_{t+1}\middle| b_t\right] ~\leq~
\left(1-c\eta\lambda\xi\right)b_t
+c'\eta^2.
\end{equation}
for some numerical constants $c,c'$.

\begin{lemma}\label{lem:recurseburn}
Let $B$ be the event that $b_{t}\leq 1-\xi$ for all $t=0,1,\ldots,T$. Then
for certain positive numerical constants $c_1,c_2,c_3$, if $\eta\leq
c_1\lambda$, then
\[
\E[b_{T}|B]\leq \left(\left(1-c_2\eta\lambda\xi\right)^{T}+c_3\frac{\eta}{\lambda\xi}\right).
\]
\end{lemma}
\begin{proof}
  Using \eqref{eq:bformburn}, we have for any $b_t$ satisfying event $B$ that
  \begin{align*}
  \E[b_{t+1}|b_t,B] = \E\left[b_{t+1}|b_t,b_{t+1}\leq 1-\xi\right]
  ~\leq~ \E[b_{t+1}|b_t] ~\leq~ \left(1-c\eta\lambda\xi\right)b_t
~+~c'\eta^2.
\end{align*}

Taking expectation over $b_t$ (conditioned on $B$), we get that
\begin{align*}
  \E[b_{t+1}|B] ~&\leq~ \E\left[\left(1-c\eta\lambda\xi\right)b_{t}
+c'\eta^2\middle| B\right]\\
&=\left(1-c\eta\lambda\xi\right)\E\left[b_{t}|B\right]
+c'\eta^2.
\end{align*}
Unwinding the recursion, we get
\begin{align*}
\E[b_{T}|B]~&\leq~\left(1-c\eta\lambda\xi\right)^{T}b_{0}+c'\eta^2\sum_{i=0}^{T-1}\left(1-c\eta\lambda\xi\right)^i\\
&\leq~\left(1-c\eta\lambda\xi\right)^{T}+c'\eta^2\sum_{i=0}^{\infty}\left(1-c\eta\lambda\xi\right)^i\\
&=~\left(1-c\eta\lambda\xi\right)^{T}+c'\eta^2\frac{1}{c\eta\lambda\xi}
~\leq~\left(1-c\eta\lambda\xi\right)^{T}+\frac{c'}{c}\frac{\eta}{\lambda\xi}.\\
\end{align*}
\end{proof}

We now turn to prove that the event $B$ assumed in \lemref{lem:recurse}
indeed holds with high probability:
\begin{lemma}\label{lem:eventburn}
  The following holds for certain positive numerical constants $c_1,c_2,c_3$:
  If $\eta\leq c_1\lambda$, then for any $\beta\in (0,1)$, if
  \begin{equation}\label{eq:eventburn}
  b_{0}+c_2 T\eta^2 +c_3\sqrt{T\eta^2\log(1/\beta)}\leq 1-\xi,
  \end{equation}
  then it holds with probability at least $1-\beta$ that
  \[
  b_t~\leq~ b_{0}+c_2 T\eta^2 +c_3\sqrt{T\eta^2\log(1/\beta)}~\leq~ 1-\xi
  \]
  for all $t=0,1,\ldots,T$.
\end{lemma}
\begin{proof}
 To prove the lemma, we analyze the stochastic process $b_{1},b_{2},\ldots,b_T$, and
 use a concentration of measure argument. First, we collect the following
 facts:
 \begin{itemize}
     \item \emph{$b_{0}\leq 1-\xi$}: This directly follows from the
         assumption stated in the lemma.
     \item \emph{$\E\left[b_{t+1}\middle| b_t\right]\leq b_t+ c'\eta^2$ for
         some constant $c'$}: By \eqref{eq:bformburn},
 \begin{align*}
 \E\left[b_{t+1}\middle| W_t\right] ~&\leq~
\left(1-c\eta\lambda\xi\right)b_t
+c'\eta^2
~\leq~ b_t+c'\eta^2.
 \end{align*}
    \item \emph{$|b_{t+1}-b_t|$ is bounded by $c\eta$ for some constant
        $c$}: Applying \lemref{lem:knn} for the case $k=1$, and assuming
        $\eta\leq 1/12$,
\[
 |b_{t+1}-b_{t}|~=~
 \left|\inner{\bv_1,\bw_{t+1}}^2-\inner{\bv,\bw_{t}}^2\right|
 ~\leq~ \frac{12\eta}{1-3\eta} \leq \frac{12 \eta}{3/4} ~=~ 16\eta.
\]
 \end{itemize}
 Armed with these facts, and using the maximal version of the Hoeffding-Azuma inequality \cite{hoeffding1963probability}, it follows that with probability at least
 $1-\beta$, it holds simultaneously for all $t=0,1,\ldots,T$ that
 \[
 b_t\leq b_{0}+c_2 T\eta^2+c_3\sqrt{T\eta^2\log(1/\beta)}
 \]
 for some constants $c_2,c_3$. If the expression is indeed less than $1-\xi$, then we
 get that $b_t\leq 1-\xi$ for all $t$, from which the lemma follows.
\end{proof}

Combining \lemref{lem:recurseburn} and \lemref{lem:eventburn}, and using
Markov's inequality, we get the following corollary:

\begin{lemma}\label{lem:combineburn}
Let confidence parameters $\beta,\gamma\in(0,1)$ be fixed. Then for some
positive numerical constants $c_1,c_2,c_3,c,c'$, if $\eta\leq c_1\lambda$ and
\[
b_{0}+c_2 T\eta^2+c_3\sqrt{T\eta^2\log(1/\beta)}\leq 1-\xi,
\]
then with probability at least $1-(\beta+\gamma)$, it holds that
\[
b_{T} \leq \frac{1}{\gamma}
\left(\left(1-c\eta\lambda\xi\right)^{T}+c'\frac{\eta}{\lambda\xi}\right).
\]
\end{lemma}%

We are now ready to prove our theorem. By \lemref{lem:combineburn}, for any
$\beta,\gamma\in\left(0,\frac{1}{2}\right)$ and any
\begin{align}
&\eta\leq \min
\left\{c_1,\frac{1}{2c'}\gamma^2\right\}\lambda\xi~~~~\text{and}~~~~
T\geq \frac{3\log(1/\gamma)}{c\eta\lambda\xi}\notag\\
&~~~~\text{such that}~~~~b_{0}+c_2 T\eta^2+c_3\sqrt{T\eta^2\log(1/\beta)}\leq 1-\xi,
\label{eq:condmeburn}
\end{align}
we get with probability at least $1-(\beta+\gamma)$ that
\[
b_{T}~\leq~
\frac{1}{\gamma}\left(\left(1-c\eta\lambda\xi\right)^{\frac{3\log(1/\gamma)}{c\eta\lambda\xi}}
+\frac{1}{2}\gamma^2\right).
\]
Using the inequality $(1-(1/x))^{ax}\leq \exp(-a)$, which holds for any $x>1$
and any $a$, and taking $x = 1/(c\eta\lambda\xi)$ and $a = 3\log(1/\gamma)$,
we can upper bound the above by
\[
\frac{1}{\gamma}\left(\exp\left(-3\log\left(\frac{1}{\gamma}\right)\right)+\frac{1}{2}\gamma^2\right)
~=~ \frac{1}{\gamma}\left(\gamma^3+\frac{1}{2}\gamma^2\right),
\]
and since we assume $\gamma< \frac{1}{2}$, this is at most $\frac{1}{2}$.
Overall, we got that with probability at least $1-\beta-\gamma$, $b_{T}\leq
\frac{1}{2}$, and therefore $1-\inner{\bv_1,\bw_T}^2\leq \frac{1}{2}$ as
required.

It remains to show that the parameter choices in \eqref{eq:condmeburn} can
indeed be satisfied. First, we fix $\xi=\frac{1}{2}\zeta$ (where we recall
that $0<\zeta\leq \inner{\bv_1,\bw_0}^2$), which trivially ensures that
$b_0=1-\inner{\bv_1,\bw_0}^2$ is at most $1-2\xi$. Moreover, suppose we pick
$\beta=\gamma$ in $(0,\exp(-1))$, and $\eta,T$ so that
\begin{equation}\label{eq:etaTburn}
\eta\leq \frac{c_*\gamma^2\lambda\xi^3}{\log^2(1/\gamma)}~~,~~ T = \left\lfloor\frac{3\log(1/\gamma)}{c'_*\eta\lambda\xi}\right\rfloor,
\end{equation}
where $c_*,c'_*$ are sufficiently small constants so that the bounds on
$\eta,T$ in \eqref{eq:condmeburn} are satisfied. This implies that the third
bound in \eqref{eq:condmeburn} is also satisfied, since by plugging in the
values / bounds of $T$ and $\eta$, and using the assumptions
$\gamma=\beta\leq \exp(-1)$ and $\xi\leq 1$, we have
\begin{align*}
  &b_{0}+c_2 T\eta^2+c_3\sqrt{T\eta^2\log(1/\gamma)}\\
  &\leq 1-2\xi+c_2\frac{3\log(1/\gamma)}{c'_*\lambda\xi}\eta+c_3\sqrt{\frac{3\log(1/\gamma)}{c'_*\lambda\xi}\eta\log(1/\gamma)}\\
  &\leq 1-2\xi+c_2\frac{3c_*\gamma^2\xi^2}{c'_*\log(1/\gamma)}+
  c_3\sqrt{\frac{3c_*\gamma^2\xi^2}{c'_*}}\\
  &\leq 1-2\xi+\left(\frac{3c_2 c_*}{c'_*}+c_3\sqrt{\frac{3c_*}{c'_*}}\right)\xi,
\end{align*}
which is less than $1-\xi$ if we pick $c_*$ sufficiently small compared to
$c'_*$.

To summarize, we get that for any $\gamma \in (0,\exp(-1))$, by picking
$\eta$ as in \eqref{eq:etaTburn}, we have that after $T$ iterations (where
$T$ is specified in \eqref{eq:etaTburn}), with probability at least
$1-2\gamma$, we get $\bw_T$ such that $1-\inner{\bv_1,\bw_T}\leq
\frac{1}{2}$. Substituting $\delta=2\gamma$ and $\zeta=2\xi$, we get that if
  \[
  \inner{\bv_1,\tilde{\bw}_0}^2\geq\zeta>0,
  \]
  and $\eta$ satisfies
  \[
  \eta\leq \frac{c_1\delta^2\lambda\zeta^3}{\log^2(2/\delta)}
  \]
  (for some universal constant $c_1$), then with probability at least $1-\delta$, after
  \[
  T ~=~ \left\lfloor\frac{c_2\log(2/\delta)}{\eta\lambda\zeta}\right\rfloor.
  \]
  stochastic iterations, we get a satisfactory point $\bw_T$.

As discussed at the beginning of the proof, this analysis is valid assuming
$r=\max_i \norm{\bx_i}^2\leq 1$. By the reduction discussed at the beginning
of \subsecref{subsec:proofmainmain}, we can get an analysis for any $r$ by
substituting $\lambda\rightarrow \lambda/r$ and $\eta\rightarrow \eta r$.
This means that we should pick $\eta$ satisfying
\[
\eta r\leq \frac{c_1\delta^2(\lambda/r)\zeta^3}{\log^2(2/\delta)}
~~~\Rightarrow~~~ \eta\leq \frac{c_1\delta^2\lambda\zeta^3}{r^2\log^2(2/\delta)},
\]
and getting the required point after
  \[
T ~=~ \left\lfloor\frac{c_2\log(2/\delta)}{(\eta
r)(\lambda/r)\zeta}\right\rfloor
~=~
\left\lfloor\frac{c_2\log(2/\delta)}{\eta\lambda\zeta}\right\rfloor
  \]
  iterations.

\subsection{Proof of \thmref{thm:convex}}\label{subsec:proofconvex}

For simplicity of notation, we drop the $_A$ subscript from $F_A$, and refer
simply to $F$.

We first prove the following two auxiliary lemmas:

\begin{lemma}\label{lem:hessian}
  If $A$ is a symmetric matrix, then the gradient of the function
  $F(\bw)=-\frac{\bw^\top A \bw}{\norm{\bw}^2}$ at some $\bw$ equals
  \[
  -\frac{2}{\norm{\bw}^2}\left(F(\bw)I+A\right)\bw,
  \]
  and its Hessian equals
  \[
  -\frac{1}{\norm{\bw}^2}\left(\left(I-\frac{4}{\norm{\bw}^2}\bw\bw^\top\right)\bigg(F(\bw)I+A\bigg)\right)^{\bot},
  \]
  where $B^\bot = B+B^\top$ (i.e., a matrix $B$ plus its transpose).
\end{lemma}
\begin{proof}
  By the product and chain rules (using the fact that $\frac{1}{\norm{\bw}^2}$ is a composition of $\bw\mapsto \norm{\bw}^2$ and $z\mapsto
  \frac{1}{z}$), the gradient of $F(\bw)=-\frac{1}{\norm{\bw}^2}\left(\bw^\top A \bw\right)$
equals
  \begin{equation}\label{eq:gggg0}
  \bw\frac{2}{\norm{\bw}^4}\left(\bw^\top A\bw\right)-\left(A\bw\right)\frac{2}{\norm{\bw}^2},
  \end{equation}
  giving the gradient bound in the lemma statement after a few
  simplifications.

  Differentiating the vector-valued \eqref{eq:gggg0} with respect to $\bw$ (using the product and chain rules, and the fact that $\frac{1}{\norm{\bw}^4}$ is a composition
  of $\bw\mapsto \norm{\bw}^2$, $z\mapsto z^2$, and $z\mapsto \frac{1}{z}$),
  we get that the Hessian of $F$ equals
  \begin{align*}
    &I \frac{2}{\norm{\bw}^4}(\bw^\top A\bw)~+~\bw\left(-\frac{2}{\norm{\bw}^8}*2\norm{\bw}^2*2\bw\right)^\top\left(\bw^\top A\bw\right)~+~\bw\frac{2}{\norm{\bw^4}}\left(2A\bw\right)^\top\\
    &~~~~~~-A\frac{2}{\norm{\bw}^2}-\left(A\bw\right)\left(-\frac{2}{\norm{\bw}^4}*2\bw\right)^\top\\
    &~=~-\frac{2F(\bw)}{\norm{\bw}^2}I~+~\frac{8F(\bw)}{\norm{\bw}^4}\bw\bw^\top~+~\frac{4}{\norm{\bw}^4}\bw\bw^\top A
    ~-~\frac{2}{\norm{\bw}^2}A~+~\frac{4}{\norm{\bw}^4}A\bw\bw^\top\\
    &=-\frac{1}{\norm{\bw}^2}\left(2F(\bw)I-\frac{8F(\bw)}{\norm{\bw}^2}\bw\bw^\top-\frac{4}{\norm{\bw}^2}\bw\bw^\top A+2A-\frac{4}{\norm{\bw}^2}A\bw\bw^\top\right),
  \end{align*}
  which can be verified to equal the expression in the lemma statement (using
  the fact that $A,\bw\bw^\top$ and $I$ are all symmetric matrices, hence
  equal their transpose).
  \end{proof}

\begin{lemma}\label{lem:geom}
  Let $\bw_0,\bv_1$ be two unit vectors such that $\norm{\bw_0-\bv_1}\leq
  \epsilon<\frac{1}{2}$ (which implies $\inner{\bw_0,\bv_1}>0$).
  Let $\bv'_1$ be the intersection of the ray $\{a\bv_1:a\geq 0\}$ with the
  hyperplane $H_{\bw_0}=\{\bw:\inner{\bw,\bw_0}=1\}$. Then
  $\norm{\bv'_1-\bw_0}\leq \frac{5}{4}\epsilon$.
\end{lemma}
\begin{proof}
  See Figure \ref{fig:const} in the main text for a graphical illustration.

  Letting $\bv'_1=a\bv$, $a$ must satisfy $\inner{a\bv_1,\bw_0}=1$. Since $\bv_1,\bw_0$
  are unit vectors, this implies
  \[
  a ~=~ \frac{1}{\inner{\bv_1,\bw_0}} ~=~ \frac{2}{2-\norm{\bv_1-\bw_0}^2},
  \]
  and since $\norm{\bv_1-\bw_0}\leq \epsilon$, this means that
  \[
  a\in \left[1,\frac{2}{2-\epsilon^2}\right].
  \]
  Therefore,
  \[
  \norm{\bv'_1-\bw_0}~\leq~ \norm{\bv_1-\bw_0}+\norm{\bv'_1-\bv_1} ~\leq~ \epsilon+\norm{a\bv_1-\bv_1} ~\leq~ \epsilon+|a-1|
  ~\leq~ \epsilon+\frac{2}{2-\epsilon^2}-1 ~=~ \epsilon+\frac{\epsilon^2}{2+\epsilon^2},
  \]
  and since $\epsilon< \frac{1}{2}$, this is at most $\frac{5}{4}\epsilon$.
\end{proof}

We now turn to prove the theorem. Let $\nabla^2(\bw)$ denote the Hessian at
some point $\bw$. To show smoothness and strong convexity as stated in the
theorem, it is enough to fix some unit $\bw_0$ which is $\epsilon$-close to
the leading eigenvector $\bv_1$ (where $\epsilon$ is assumed to be
sufficiently small), and show that for any point $\bw$ on $H_{\bw_0}$ which
is $\Ocal(\epsilon)$ close to $\bw_0$, and any direction $\bg$ along
$H_{\bw_0}$ (i.e. any unit $\bg$ such that $\inner{\bg,\bw_0}=0$), it holds
that $\bg^\top \nabla^2(\bw)\bg\in [\lambda,20]$. This implies that the
second derivative in an $\Ocal(\epsilon)$ neighborhood of $\bw_0$ on
$H_{\bw_0}$ is always in $[\lambda,20]$, hence the function is both
$\lambda$-strongly convex in that neighborhood.

More formally, letting $\epsilon\in (0,1)$ be a small parameter to be chosen
later, consider any $\bw_0$ such that
\[
\norm{\bw_0}=1~~,~~\norm{\bw_0-\bv_1} \leq \epsilon,
\]
any $\bw$ such that
\[
\inner{\bw-\bw_0,\bw_0}=0~~,~~\norm{\bw-\bw_0}\leq 2\epsilon,
\]
and any $\bg$ such that
\[
\norm{\bg}=1~~,~~\inner{\bg,\bw_0}=0.
\]
Our goal is to show that for an appropriate $\epsilon$, we have $\bg^\top
\nabla^2(\bw)\bg\in [\lambda,20]$. Moreover, by \lemref{lem:geom}, the
neighborhood set $H_{\bw_0}\cap B_{\bw_0}(2\epsilon)$ would also contain a
point $a\bv_1$ for some $a$, which is a global optimum of $F$ due to its
scale-invariance. This would establish the theorem.

The easier part is to show the upper bound on $\bg^\top \nabla^2(\bw)\bg$.
Since $\bg$ is a unit vector, it is enough to bound the spectral norm of
$\nabla^2(\bw)$, which equals
\begin{align*}
  &\left\|\frac{1}{\norm{\bw}^2}\left(\left(I-\frac{4}{\norm{\bw}^2}\bw\bw^\top\right)\bigg(F(\bw)I+A\bigg)\right)^{\bot}\right\|_{sp}\\
  &\leq \frac{2}{\norm{\bw}^2}\left\|\left(I-\frac{4}{\norm{\bw}^2}\bw\bw^\top\right)\bigg(F(\bw)I+A\bigg)\right\|_{sp}\\
  &\leq \frac{2}{\norm{\bw}^2}\left\|I-\frac{4}{\norm{\bw}^2}\bw\bw^\top\right\|_{sp}\norm{F(\bw)I+A}_{sp}\\
  &\leq \frac{2}{\norm{\bw}^2}\left(\norm{I}_{sp}+\left\|\frac{4}{\norm{\bw}^2}\bw\bw^\top\right\|_{sp}\right)\left(\norm{F(\bw)I}_{sp}+\norm{A}_{sp}\right).
\end{align*}
Since the spectral norm of $A$ is $1$, and $\norm{\bw}^2\geq 1$ (as $\bw$
lies on a hyperplane $H_{\bw_0}$ tangent to a unit vector $\bw_0$), it is
easy to verify that this is at most $2(1+4)(1+1) = 20$ as required.

We now turn to lower bound $\bg^\top \nabla^2(\bw)\bg$, which by
\lemref{lem:hessian} equals
\[
-\frac{1}{\norm{\bw}^2}\bg^\top\left(\left(I-\frac{4}{\norm{\bw}^2}\bw\bw^\top\right)\bigg(F(\bw)I+A\bigg)\right)^{\bot}\bg.
\]
Since $\bg^\top B^\bot \bg = \bg^\top B \bg + \bg^\top B^\top \bg = 2\bg^\top
B \bg$, the above equals
\begin{equation}\label{eq:gggg1}
-\frac{2}{\norm{\bw}^2}\bg^\top\left(I-\frac{4}{\norm{\bw}^2}\bw\bw^\top\right)\bigg(F(\bw)I+A\bigg)\bg.
\end{equation}
Using the fact that $\bw=\bw_0+(\bw-\bw_0)$, and $\inner{\bg,\bw_0}=0$, we
get that $\inner{\bg,\bw} = \inner{\bg,\bw-\bw_0}$. Moreover, since $A$ is
positive semidefinite and has spectral norm of $1$, $F(\bw)=-\frac{\bw^\top
A\bw}{\norm{\bw}^2}\in [-1,0]$. Expanding \eqref{eq:gggg1} and plugging these
in, we get
\begin{align*}
  &-\frac{2}{\norm{\bw}^2}\left(F(\bw)\bg^\top\left(I-\frac{4}{\norm{\bw}^2}\bw\bw^\top\right)\bg+
  \bg^\top\left(I-\frac{4}{\norm{\bw}^2}\bw\bw^\top\right)A\bg\right)\\
  &=\frac{2}{\norm{\bw}^2}\left(-F(\bw)\norm{\bg}^2+\frac{4F(\bw)}{\norm{\bw}^2}\inner{\bg,\bw-\bw_0}^2~-~\bg^\top A \bg+\frac{4}{\norm{\bw}^2}\inner{\bg,\bw-\bw_0}\bw^\top A\bg\right)\\
  &\geq \frac{2}{\norm{\bw}^2}\left(-F(\bw)\norm{\bg}^2-\frac{4}{\norm{\bw}^2}\norm{\bg}^2\norm{\bw-\bw_0}^2~-~\bg^\top A \bg-\frac{4}{\norm{\bw}^2}\norm{\bg}\norm{\bw-\bw_0}\norm{\bw}\norm{A}_{sp}\norm{\bg}\right).
\end{align*}
Since $\norm{\bg}=1$, $\norm{A}_{sp}= 1$, $\norm{\bw-\bw_0}\leq 2\epsilon$,
and $\norm{\bw}^2=\norm{\bw_0}^2+\norm{\bw-\bw_0}^2$ is between $1$ and
$1+4\epsilon^2$, this is at least
\begin{equation}\label{eq:gggg2}
\frac{2}{\norm{\bw}^2}\left((-F(\bw))-16\epsilon^2-\bg^\top A
\bg-8\epsilon\sqrt{1+4\epsilon^2}\right) ~=~
\frac{2}{\norm{\bw}^2}\left(-F(\bw)-\bg^\top A
\bg-8\epsilon\left(2\epsilon+\sqrt{1+4\epsilon^2}\right)\right).
\end{equation}

Let us now analyze $-F(\bw)$ and $\bg^\top A \bg$ more carefully. The idea
will be to show that since we are close to the optimum, $-F(\bw)$ is very
close to $1$, and $\bg$ (which is orthogonal to the near-optimal $\bw_0$) is
such that $\bg^\top A \bg$ is strictly smaller than $1$. This would give us a
positive lower bound on \eqref{eq:gggg2}.
\begin{itemize}
\item By the triangle inequality and the assumptions
    $\norm{\bw_0-\bv_1}\leq \epsilon$, $\norm{\bw-\bw_0}\leq 2\epsilon$, we
    have $\norm{\bw-\bv_1}\leq 3\epsilon$. Also, we claim that $F(\cdot)$
    is $4$-Lipschitz outside the unit Euclidean ball (since the gradient of
    $F$ at any point with norm $\geq 1$, according to \lemref{lem:hessian},
    has norm at most $4$). Therefore, $|F(\bw)+1|=|F(\bw)-F(\bv_1)| \leq
    4\norm{\bw-\bv_1}\leq  12\epsilon$, so overall,
    \begin{equation}\label{eq:fbw}
    F(\bw)\leq -1+12\epsilon.
    \end{equation}
\item Since $\inner{\bw_0,\bg}=0$, and $\norm{\bw_0-\bv_1}\leq \epsilon$,
    it follows that
    \[
    |\inner{\bv_1,\bg}|~\leq~ |\inner{\bv_1-\bw_0,\bg}|+|\inner{\bw_0,\bg}|~\leq~ \norm{\bv_1-\bw_0}\norm{\bg}+0 ~\leq~ \epsilon.
    \]
    Letting $\bv_1,\ldots,\bv_d$ and $1=s_1> s_2\geq..\geq s_d\geq 0$ be
    the eigenvectors and eigenvalues of $A$ in decreasing order (and
    recalling that $s_2\leq s_1-\lambda = 1-\lambda$ for some eigengap
    $\lambda>0$), we get
    \begin{align}
    \bg^\top A \bg &= \sum_{i=1}^{d}s_i\inner{\bv_i,\bg}^2 ~\leq~ \inner{\bv_1,\bg}^2+(1-\lambda)\sum_{i=1}^{d}\inner{\bv_i,\bg}^2\notag\\
    &=~ \inner{\bv_1,\bg}^2+(1-\lambda)(1-\inner{\bv_1,\bg}^2)~=~ \lambda\inner{\bv_1,\bg}^2+(1-\lambda)\notag\\
    &\leq \lambda\epsilon^2+(1-\lambda)~=~ 1-(1-\epsilon^2)\lambda.\label{eq:gag}
    \end{align}
\end{itemize}
Plugging \eqref{eq:fbw} and \eqref{eq:gag} back into \eqref{eq:gggg2}, we get
a lower bound of
\begin{align*}
&\frac{2}{\norm{\bw}^2}\left(1-12\epsilon-\left(1-(1-\epsilon^2)\lambda\right)-8\epsilon\left(2\epsilon+\sqrt{1+4\epsilon^2}\right)\right)\\
&=\frac{2}{\norm{\bw}^2}\left((1-\epsilon^2)\lambda-8\epsilon\left(1.5+2\epsilon+\sqrt{1+4\epsilon^2}\right)\right)\\
&=\frac{2}{\norm{\bw}^2}\left(1-\epsilon^2-\frac{8\epsilon\left(1.5+2\epsilon+\sqrt{1+4\epsilon^2}\right)}{\lambda}\right)\lambda.
\end{align*}
Using the fact that $\sqrt{1+z^2}\leq 1+z$, this can be loosely lower bounded
by
\[
\frac{2}{\norm{\bw}^2}\left(1-\epsilon-\frac{8\epsilon\left(2.5+4\epsilon\right)}{\lambda}\right)\lambda.
\]
 Recalling that $\norm{\bw}^2=\norm{\bw_0}^2+\norm{\bw-\bw_0}^2$ is at most
$1+4\epsilon^2$, and picking $\epsilon$ sufficiently small compared to
$\lambda$, (say $\epsilon= \lambda/44$), we get that the above is at least
$\lambda$, which implies the required strong convexity condition.

To summarize, by picking $\epsilon=\lambda/44$, we have shown that the
function $F(\bw)$ is $\lambda$-strongly convex and $20$-smooth in a
neighborhood of size $2\epsilon=\frac{\lambda}{22}$ around $\bw_0$ on the
hyperplane $H_{\bw_0}$, provided that $\norm{\bw_0-\bv_1}\leq \epsilon =
\frac{\lambda}{44}$. By \lemref{lem:geom}, we are guaranteed that this
neighborhood contains $\bv_1$ up to some rescaling (which is immaterial for
our scale-invariant function $F$), hence by optimizing $F$ in that
neighborhood, we will get a globally optimal solution.


\subsubsection*{Acknowledgments}
This research is supported in part by an FP7 Marie Curie CIG grant, the Intel
ICRI-CI Institute, and Israel Science Foundation grant 425/13.

\bibliographystyle{plain}
\bibliography{mybib}

\end{document}